\documentclass{article}
\usepackage{fullpage}

\usepackage[numbers]{natbib}

\usepackage[utf8]{inputenc} 
\usepackage[T1]{fontenc}    
\usepackage{hyperref}       
\usepackage{url}            
\usepackage{booktabs}       
\usepackage{amsfonts}       
\usepackage{nicefrac}       
\usepackage{microtype}      

\usepackage{amsmath}
\usepackage{amssymb}
\usepackage{amsthm}
\usepackage{appendix}
\usepackage{cleveref}
\usepackage{mathtools}

\usepackage{tikz-cd}

\usepackage{xr}

\usepackage{sidecap}
\usepackage{enumitem}

\usepackage{listings}

\usepackage{float}

\DeclareMathOperator*{\argmin}{arg\,min}
\DeclareMathOperator*{\argmax}{arg\,max}

\newcommand\Perm[1]{\mathfrak{S}_{#1}}

\newcommand\LL{\textit{\L}}
\newcommand\SkCZ{ \Perm{k}\mathcal{C}_{\mathbb{Z}}}
\newcommand\signvect{\mathfrak{s}}

\newtheorem{theorem}{Theorem}[section]
\crefname{theorem}{Theorem}{Theorems}
\newtheorem{lemma}[theorem]{Lemma}
\crefname{lemma}{Lemma}{Lemmas}

\newtheorem{corollary}[theorem]{Corollary}
\crefname{corollary}{Corollary}{Corollaries}
\newtheorem{proposition}[theorem]{Proposition}
\crefname{proposition}{Proposition}{Propositions}

\crefname{conjecture}{Conjecture}{Conjecture}

\theoremstyle{definition}

\newtheorem{definition}[theorem]{Definition}
\crefname{definition}{Definition}{Definitions}
\crefname{figure}{Figure}{Figures}
\crefname{SCfigure}{Figure}{Figures}
\crefname{section}{Section}{Sections}
\crefname{subsection}{Section}{Sections}

\crefname{appendix}{Appendix}{Appendices}

\theoremstyle{remark}

\title{Weston-Watkins Hinge Loss and Ordered Partitions}

\author{%
  Yutong Wang \\
  University of Michigan\\
  \texttt{yutongw@umich.edu} \\
  \and
  Clayton D. Scott\\
  University of Michigan\\
  \texttt{clayscot@umich.edu} \\
}
\date{}

\begin{document}

\maketitle

\begin{abstract}
Multiclass extensions of the support vector machine (SVM) have been formulated in a variety of ways.
A recent empirical comparison of nine such formulations \cite{dogan2016unified} recommends the variant proposed by Weston and Watkins (WW), despite the fact that the WW-hinge loss is not calibrated with respect to the 0-1 loss.
In this work we introduce a novel discrete loss function for multiclass classification, the \emph{ordered partition loss}, and prove that the WW-hinge loss \emph{is} calibrated with respect to this loss.
We also argue that the ordered partition loss is maximally informative among discrete losses satisfying this property.
Finally, we apply our theory to justify the empirical observation made by \citet{dogan2016unified} that the WW-SVM can work well even under massive label noise, a challenging setting for multiclass SVMs.
\end{abstract}

\section{Introduction}
Classification is the task of assigning labels to instances, and a common approach is to minimize misclassification error corresponding to the 0-1 loss. However, the 0-1 loss is discrete and typically cannot be optimized efficiently.
To address this, the 0-1 loss is often replaced by a surrogate loss during training.
If the surrogate is \emph{calibrated} with respect to the 0-1 loss, then a classifier minimizing the expected surrogate loss will also minimize the expected 0-1 loss in the infinite sample limit.

For multiclass classification, several different multiclass extensions of the support vector machine (SVM) have been proposed, including the Weston-Watkins (WW) \cite{weston1999support},
Crammer-Singer (CS) \cite{crammer2001algorithmic}, and
Lee-Lin-Wahba (LLW) \cite{lee2004multicategory} SVMs.
The pertinent difference between these multiclass SVMs is the multiclass generalization of the hinge loss.
Below, we refer to the hinge loss from WW-SVM as the WW hinge loss and so on.
It is well-known that the LLW-hinge is calibrated with respect to the 0-1 loss, while the WW- and CS-hinge losses are not \cite{liu2007fisher, tewari2007consistency}.

Despite this result, the LLW-SVM is not more widely accepted than the WW-, CS-, and other SVMs.
The first reason for this is that while the LLW-SVM is calibrated with respect to the 0-1 loss, this did not lead to superior performance empirically.
In particular, \citet{dogan2016unified} found that the LLW-SVM fails in low dimensional feature space even under the noiseless setting.
On the other hand, \citet{dogan2016unified} observed that the WW-SVM is the only multiclass SVM that succeeded in both the noiseless and noisy setting in their simulations.
Indeed, \citet{dogan2016unified} concluded that, among 9 different competing multiclass SVMs, the WW-SVM offers the best overall performance when considering accuracy and computation.
The second reason is that the calibration framework is not limited to the 0-1 loss.
There could be other discrete losses with respect to which a surrogate is calibrated, and which help to explain its performance.
Indeed, \citet{ramaswamy2018consistent} recently showed that the CS-hinge loss is calibrated with respect to a discrete loss for classification with abstention.

In a vein similar to \cite{ramaswamy2018consistent}, we show that the WW-hinge loss
is calibrated with respect to a novel discrete loss that we call the \emph{ordered partition} loss.
Our results leverage the embedding framework for analyzing discrete losses and convex piecewise linear surrogates, introduced recently by \citet{finocchiaro2019embedding}.
We also give theoretical justification for the empirical performance of the WW-SVM observed by \citet{dogan2016unified}.

\subsection{Related work}

\citet{cortes1995support} introduced the support vector machine for learning a binary classifier, using the hinge loss as a surrogate for the 0-1 loss.
\citet{steinwart2002support} showed that the binary SVM is \emph{universally consistent}, a desirable property of a classification algorithm that ensures its convergence to the Bayes optimal classifier in the large sample limit.
\citet{steinwart2005consistency} later used calibration to give a more general proof of SVM consistency with respect to the 0-1 loss. 
Around that time, more general theories of when a loss is calibrated with respect to 0-1 loss, or ``classification calibrated," began to emerge \cite{zhang2004statistical, bartlett2006convexity, steinwart2007compare}, and since then a proliferation of papers have extended these ideas to a variety of learning settings (see \citet{bao2020calibrated} for a recent review).

Several natural extensions of the binary SVM exist,
including 
the Weston-Watkins (WW) \cite{weston1999support},
Crammer-Singer (CS) \cite{crammer2001algorithmic}, and
Lee-Lin-Wahba (LLW) \cite{lee2004multicategory} SVMs.
\citet{tewari2007consistency} extended the definition of calibration with respect to the 0-1 loss to the multiclass setting.
\citet{liu2007fisher}
and
\citet{tewari2007consistency}
analyzed these hinge losses and showed that 
WW and CS hinge losses are not calibrated with respect to the 0-1 loss while 
the LLW hinge loss is.
\citet{dogan2016unified} introduced a framework that unified existing multiclass SVMs, proved the 0-1 loss consistency of several multiclass SVMs when the kernel is allowed to change, and also conducted extensive experiments.
Despite not being calibrated with respect to the 0-1 loss, \citet{zhang2004statistical} showed that 
the Crammer-Singer SVM is consistent given the ``majority assumption'',
i.e., 
the most probable class has greater than $1/2$ probability.
When the majority assumption is violated, experiments conducted by \citet{dogan2016unified} suggested that 
the CS-SVM fails, while the WW-SVM continues to perform well.

\citet{ramaswamy2016convex} extended the notion of calibration to an arbitrary discrete loss used in \emph{general multiclass learning}.
The general multiclass learning framework unifies several learning problems,
including cost-sensitive classification \cite{scott2012calibrated},
classification with abstain option \cite{ramaswamy2018consistent},
ranking \cite{duchi2013asymptotics},
and partial label learning \cite{cid2012proper}.  
Furthermore, \citet{ramaswamy2016convex} introduced the concept of \emph{convex calibration dimension} which is defined for a discrete loss to be the minimum dimension required for the domain of a convex surrogate loss to be calibrated with respect to the given discrete loss.
\citet{ramaswamy2018consistent} proved the consistency of CS-SVM with respect to the abstention loss where the cost of abstaining is $1/2$ by showing that the CS hinge is calibrated with respect to this abstention loss.
They also proposed a new calibrated convex surrogate loss in dimension $\lceil \log_2 k \rceil$ for the abstention loss, implying that the CS hinge is suboptimal from the CC-dimension perspective.

Recently, several new multiclass hinge-like losses have been proposed, as well as frameworks for constructing convex losses. \citet{dogan2016unified} used their framework to devise two new multiclass hinge losses, and using ideas from adversarial multiclass classification, \citet{fathony2016adversarial} proposed a new multiclass hinge-like loss;  all three are calibrated with respect to the 0-1 loss. \citet{blondel2019learning} introduced a class of losses known as \emph{Fenchel-Young losses} which contains non-smooth losses such as the CS hinge loss as well as smooth losses such as the logistic loss.
\citet{tan2020loss} proposed an approach for constructing hinge-like losses using generalized entropies. \citet{finocchiaro2019embedding} studied the calibration properties of \emph{polyhedral} losses using the \emph{embedding} framework that they developed. They analyzed several polyhedral losses in the literature including the CS hinge, the Lov{\'a}sz hinge \cite{yu2018lovasz}, and the top-$n$ loss \cite{lapin2017analysis}.

\subsection{Our contributions}

We introduce a novel discrete loss $\ell$, the \emph{ordered partition loss}.
We show in \cref{main-theorem: embedding result only} that
the Weston-Watkins hinge loss $L$ embeds the ordered partition loss $\ell$.
Our embedding result together with results of \cite{finocchiaro2019embedding} imply 
that $L$ is calibrated with respect to $\ell$ (\cref{main-corollary: calibration result only}).
To the best of our knowledge, this is the first calibration-theoretic result for the WW-hinge loss.
We also introduce the notion of the \emph{maximally informative} discrete loss that a polyhedral loss can embed and argue that the ordered partition loss is maximally informative for the WW-hinge loss.  In \cref{main-section: the argmax link}, we use properties of the ordered partition loss to give theoretical support for the empirical observations made by \citet{dogan2016unified} on the success of WW-SVM in the massive label noise setting.

\subsection{Notations}\label{main-section: notations}
Let $k \ge 3$ be an integer which denotes the number of classes.
For a positive integer $n$, we let $[n] = \{1,\dots, n\}$.
If $v = (v_1,\dots, v_k) \in \mathbb{R}^k$ and $i \in [k]$ is an index, then let $[v]_i := v_i$.
Define $\max v = \max_{i \in [k]} v_i$ and $\argmax v= \{i \in [k]: v_i = \max v\}$.

Let $\Perm{k}$ denote the set of permutations on $[k]$, i.e., elements of $\Perm{k}$ are bijections $\sigma: [k] \to [k]$.
Given $\sigma \in \Perm{k}$ and $v \in \mathbb{R}^k$, the vector $\sigma v \in \mathbb{R}^k$ is defined entrywise where the $i$-th entry is $[\sigma v]_i = v_{\sigma(i)}$.
Equivalently, we view $\Perm{k}$ as the set of permutation matrices in $\mathbb{R}^{k\times k}$.

Let $\mathbb{R}_+$ denote the set of nonnegative reals.
Denote $\Delta^k = \{(p_1,\dots, p_k) \in \mathbb{R}^k_+ : p_1 + \cdots +p_k = 1\}$ the probability simplex.
For $p \in \Delta^k$, we write $Y\sim p$ to denote a discrete random variable $Y \in [k]$ whose probability mass function is $p$.
Let $\langle\cdot,\, \cdot \rangle$ be the usual dot-product between vectors.
Denote by $\mathbb{I}\{\verb|input|\}$ the indicator function which returns $1$ if $\verb|input|$ is true and $0$ otherwise.

\subsection{Background}\label{main-section: backgrounds}

Recall the \emph{general multiclass learning} framework as described in \cite{ramaswamy2016convex}: 
$\mathcal{X}$ is a sample space and $P$ is a joint distribution over $\mathcal{X} \times [k]$.
A \emph{multiclass classification loss} is a function $\ell: \mathcal{R} \to \mathbb{R}^k_+$ where $\mathcal{R}$ is called the \emph{prediction space} and $[\ell(r)]_y \in \mathbb{R}_+$ is the penalty incurred for predicting $r \in \mathcal{R}$ when the label is $y \in [k]$.
If $\mathcal{R}$ is finite, we refer to $\ell$ as a \emph{discrete loss}.
For example, a common setting for classification is $\mathcal{R} = [k]$ and $\ell$ is the 0-1 loss.
The $\ell$-\emph{risk} of a \emph{hypothesis} function $f: \mathcal{X} \to \mathcal{R}$ is 
\begin{equation}
  \label{main-equation: ell risk}
  \mathrm{er}^{\ell}_P(f) :=
  \mathbb{E}_{X,Y \sim P} 
  \left\{
    [\ell(f(X))]_Y
  \right\}.
\end{equation}
The goal is to design \emph{$\ell$-consistent} algorithms, i.e., procedures that output a hypothesis $f_n$ based on an input of $n$ training samples sampled i.i.d from $P$
such that
$\mathrm{er}_P^\ell(f_n)
\to \mathrm{er}_P^{\ell, *} = 
  \inf_{f :\mathcal{X} \to \mathcal{R}}
  \mathrm{er}_P^{\ell}(f)
  $ as $n\to \infty$.
Since $\ell$ is discrete, \cref{main-equation: ell risk} is difficult to directly minimize.
To circumvent this difficulty, we consider a convex \emph{surrogate loss} $L : \mathbb{R}^d \to \mathbb{R}^k$ for some positive integer $d$.
The following property relates the surrogate loss $L$ and the discrete loss $\ell$.
\begin{definition}[Calibration]\label{main-definition: calibration}
For each $p \in \Delta^k$, define $\gamma_\ell(p) := \argmin_{r \in \mathcal{R}} \langle p, \ell(r)\rangle$.
We say that \emph{$L$ is calibrated with respect to $\ell$} if there exists a function $\psi: \mathbb{R}^d \to \mathcal{R}$  such that for all $p \in \Delta^k$
  \[
    \inf_{u \in \mathbb{R}^d : \psi(u) \not \in \gamma_\ell(p)}
    \langle p, L(u)\rangle
    >
    \inf_{v \in \mathbb{R}^d }
    \langle p, L(v)\rangle.
  \]
\end{definition}
By \citet[Theorem 3]{ramaswamy2016convex}, $L$ being calibrated with respect to $\ell$ is equivalent to the following:
there exists $\psi: \mathbb{R}^d \to \mathcal{R}$ such that
for all joint distributions $P$ on $\mathcal{X} \times [k]$ and all sequences of functions $g_n: \mathcal{X} \to \mathbb{R}^d$, we have
\[
  \mathrm{er}_P^L(g_n)
\to
\mathrm{er}_P^{L,*}
  \quad \mbox{implies} \quad
  \mathrm{er}_P^\ell(\psi \circ g_n)
\to
  \mathrm{er}_P^{\ell,*}
\]
where 
$
\mathrm{er}_P^{L,*}
=
\inf_{g: \mathcal{X} \to \mathbb{R}^d} \mathrm{er}_P^{L}(g)$.
Thus, the calibration property allows us to focus on finding $L$-consistent algorithms.
In general it can be difficult to check that a given $L$ is calibrated with respect to $\ell$.
\citet{finocchiaro2019embedding} introduced the following definition:
\begin{definition}
  [\citet{finocchiaro2019embedding}]
  \label{main-definition: loss embedding}
  The loss $L : \mathbb{R}^d \to \mathbb{R}^k$ \emph{embeds} $\ell : \mathcal{R} \to \mathbb{R}^k$ if there exists an injection $\varphi : \mathcal{R} \to \mathbb{R}^d$ called an \emph{embedding} such that
  \begin{enumerate}
    \item \label{main-definition: loss embedding - loss vector are equal}
      $L(\varphi(r)) = \ell(r)$ for all $r \in \mathcal{R}$
    \item \label{main-definition: loss embedding - argmin are equivalent}
      $r \in \argmin_{r \in \mathcal{R}} \langle p, \ell(r)\rangle
      $ if and only if
      $\varphi(r) \in \argmin_{v \in \mathbb{R}^d} \langle p, L(v)\rangle$.
  \end{enumerate}
\end{definition}
The notion of embedding is important due to the following result 
from \cite[Theorem 3]{finocchiaro2019embedding}:
\begin{theorem}
  [\citet{finocchiaro2019embedding}]
  \label{main-theorem: embedding implies calibration}
Let $L$ be convex piecewise-linear and $\ell$ be discrete. If $L$ embeds $\ell$, then $L$ is calibrated with respect to $\ell$.
\end{theorem}
\citet{finocchiaro2019embedding} also provide an explicit construction for $\psi$ given $L,\ell$ and $\varphi$.

In this work, we are interested in the case when $L$ is the WW-hinge loss:
  \begin{definition}
    \label{main-definition: full WW hinge loss}
    For $v \in \mathbb{R}^{k}$, define the
\emph{Weston-Watkins hinge loss}
\cite{weston1999support}
$L(v) \in \mathbb{R}^{k}_+$ entrywise by
\[
  [L(v)]_y = \sum_{i \in [k]\,:\, i \ne y} 
  h(v_y - v_i), \qquad y \in [k]
\]
where $h : \mathbb{R} \to \mathbb{R}_+$ is the \emph{hinge function} defined by $h(x) = \max\{0,1-x\}$.
  \end{definition}
  By \cref{main-theorem: embedding implies calibration}, to prove that $L$ is calibrated with respect to $\ell$, it suffices to show that $L$ embeds $\ell$.
  Going forward, $L$ will refer to the WW-hinge loss.
  We now work toward showing that $L$ embeds the ordered partition loss $\ell$, which we introduce next.

\begin{SCfigure}
  \label{main-figure: Bayes optimal classifier of the OP loss}
  \centering
  \caption{
    The gray triangle represents the probability simplex $\Delta^3$, where $(p_1,p_2,p_3) \in \Delta^3$ is plotted as $(p_2,p_3)$ in the plane.
    The interior of each polygonal region contains $p \in \Delta^3$ such that $\min_{\mathbf{S} \in \mathcal{OP}_k} \langle p, \ell(\mathbf{S})\rangle$ has a unique minimizer.
    For the derivations, see \cref{section: figure 1 generating code}.
    Ordered partitions are represented as follows:\newline 
    \smallskip$\quad(\{1\},\{2,3\}) \mapsto 1 | 23$,\newline
    \smallskip$\quad(\{1\},\{2\},\{3\})\mapsto1|2|3$,\newline
    \smallskip $\qquad\vdots$\newline
    \smallskip$\quad(\{3\},\{2\},\{1\}) \mapsto 3|2|1$.
  }
  \begin{tikzpicture}[scale = 0.75]
    \fill[black!5!white] (0,0) -- (6,0) -- (0,6) -- cycle;
    \draw[help lines, dashed] (0,0) grid (6,6);
    \draw (0,6) -- (6,0);
    \node[draw,circle,inner sep=2pt,fill] at (0,0) {};
    \node[below, scale = 0.8] at (-0.5,-0.1) {$(1,0,0)$};
    \node[draw,circle,inner sep=2pt,fill] at (0,6) {};
    \node[left, scale = 0.8] at (-0.1,6) {$(0,0,1)$};
    \node[draw,circle,inner sep=2pt,fill] at (2,2) {};
    \node[draw,circle,inner sep=2pt,fill] at (6,0) {};
    \node[below, scale = 0.8] at (6,-.1) {$(0,1,0)$};
    \node[below] at (1,0) {$\frac{1}{6}$};
    \node[below] at (2,0) {$\frac{1}{3}$};
    \node[below] at (3,0) {$\frac{1}{2}$};
    \node[above] at (7,0) {$p_2$};
    \node[right] at (0.2,6.5) {$p_3$};
    \draw [<->] (0,6.5) -- (0,0) -- (6.5,0);
    \node[left] at (0,1) {$\frac{1}{6}$};
    \node[left] at (0,2) {$\frac{1}{3}$};
    \node[left] at (0,3) {$\frac{1}{2}$};
    \draw (0,0) -- (2,1);
    \draw (2,2) -- (2,1);
    \draw (3,0) -- (2,1);
    \draw (6,0) -- (3,1);
    \draw (2,2) -- (3,1);
    \draw (3,0) -- (3,1);
    \draw (6,0) -- (3,2);
    \draw (2,2) -- (3,2);
    \draw (3,3) -- (3,2);
    \draw (0,6) -- (2,3);
    \draw (2,2) -- (2,3);
    \draw (3,3) -- (2,3);
    \draw (0,6) -- (1,3);
    \draw (2,2) -- (1,3);
    \draw (0,3) -- (1,3);
    \draw (0,0) -- (1,2);
    \draw (2,2) -- (1,2);
    \draw (0,3) -- (1,2);
    \node at (1.5,1.5) {$1|23$};
    \node at (1.75,1/3) {$1|2|3$};
    \node[rotate = -45] at (2.5,1) {$12|3$};
    \node at (3.75,1/3) {$2|1|3$};
    \node[rotate = -45] at (3,1.5) {$2|13$};
    \node[rotate = -45] at (3.6, 2) {$2|3|1$};
    \node[rotate = -45] at (2, 3.6) {$3|2|1$};
    \node at (2.5,2.5) {$23|1$};
    \node[rotate = -45] at (1.5,3) {$3|12$};
    \node[rotate = -45] at (1,2.5) {$13|2$};
    \node[rotate = -90] at (0.3,3.8) {$3|1|2$};
    \node[rotate = -90] at (0.3,1.8) {$1|3|2$};
  \end{tikzpicture} 
\end{SCfigure}
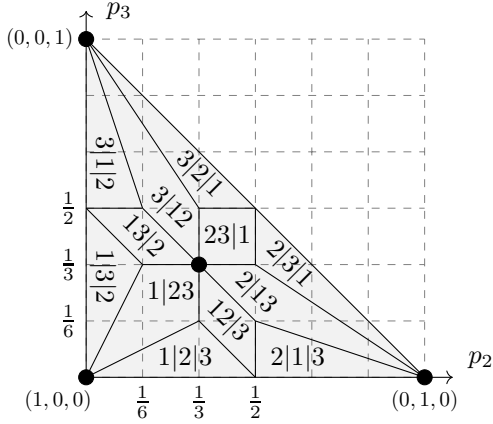
  \section{The ordered partition loss}\label{main-section: ordered partitions}

The prediction space $\mathcal{R}$ that we use is the set of ordered partitions, which we now define:
\begin{definition}
  \label{main-definition: ordered partitions}
  An \emph{ordered partition} on $[k]$ of length $l$ is an ordered list $\mathbf{S} = (S_1,\dots, S_l)$ of nonempty, pairwise disjoint subsets of $[k]$ such that $S_1 \cup \cdots \cup S_l = [k]$.
  Denote by $\mathcal{OP}_k$ the set of all ordered partitions on $[k]$ with length $\ge 2$.
  We write the length of $\mathbf{S}$ as $l_{\mathbf{S}}$ to be precise when working with multiple ordered partitions.
\end{definition}
Ordered partitions can be thought of as a complete ranking of $k$ items where ties are allowed.
They are widely studied in combinatorics \cite{mansour2012combinatorics, gross1962preferential, ishikawa2008euler}.
In the ranking literature, ordered partitions are called \emph{bucket orders}  \cite{fagin2004comparing} and the $S_i$s are called the \emph{buckets}.
The first bucket $S_1$ contains the highest ranked items, and so on.
  There is only one ordered partition with $l_{\mathbf{S}} = 1$, namely the \emph{trivial partition} $\mathbf{S} = ([k])$. Thus, $\mathcal{OP}_k$ is the set of nontrivial ordered partitions.

  We now define the following discrete loss over the ordered partitions:

\begin{definition}
  \label{main-definition: ordered partition loss}
  The \emph{ordered partition loss}
$
  \ell : \mathcal{OP}_k \to \mathbb{R}^k_+
$
is defined, for $i \in [k]$ and $\mathbf{S} = (S_1,\dots, S_l) \in \mathcal{OP}_k$, as
$
  [\ell(\mathbf{S})]_i
  =
|S_1|-1
+\sum_{j=1}^{l_{\mathbf{S}}-1}
|S_1 \cup \dots \cup S_{j+1}|
\cdot
\mathbb{I}\{
i \not \in S_1 \cup \cdots \cup S_j\}.
$
\end{definition}
To build intuitions about $\ell$, let $Y \sim p$ and consider the random variable $[\ell(\mathbf{S})]_Y$ whose expectation is
\begin{equation}
  \label{main-equation: ordered partition inner risk}
  \mathbb{E}_{Y \sim p} \left\{[\ell(\mathbf{S})]_Y\right\}
  =
|S_1|-1
+\sum_{j=1}^{l_{\mathbf{S}}-1}
|S_1 \cup \dots \cup S_{j+1}|
\cdot
\Pr_{Y \sim p}
\left\{Y \not \in S_1 \cup \cdots \cup S_j\right\}.
\end{equation}
Note that $
  \mathbb{E}_{Y \sim p} \left\{[\ell(\mathbf{S})]_Y\right\}
  =
  \langle p, \ell(\mathbf{S})\rangle
  $.
In \cref{main-figure: Bayes optimal classifier of the OP loss}, we visualize the decision rule for the Bayes optimal classifier in the $k=3$ case by plotting the function $p \mapsto \argmin_{\mathbf{S} \in \mathcal{OP}_k} 
  \langle p, \ell(\mathbf{S})\rangle
$.
When $l_\mathbf{S} = 2$, we have
\begin{equation}
  \mathbb{E}_{Y \sim p} \left\{[\ell(\mathbf{S})]_Y\right\}
  =
|S_1|-1
+
k
\Pr_{Y \sim p}
\left\{Y \not \in S_1 \right\}.
\end{equation}
Thus, we have a trade-off where adding elements to $S_1$ increases the $|S_1| - 1$ term but decreases the $
k
\Pr_{Y \sim p}
\left\{Y \not \in S_1 \right\}$ term.
More generally, when $l_{\mathbf{S}} \ge 2$, the ordered partition loss requires the predictor to associate each test instance $x$ with a nested sequence of sets $S_1,S_1\cup S_2,\cdots$ where these sets are designed to balance the probability of containing $x$'s label with the size of the set.
In the learning with partial labels settings \cite{cour2011learning, cid2012proper}, for each training instance the learner observes a set of labels, one of which is the true label.
The sets $S_1,S_1\cup S_2,\dots$ might be called \emph{progressive partial labels} in the spirit of partial label learning \cite{cour2011learning, cid2012proper}.

Next, we define the embedding that satisfies  \cref{main-definition: loss embedding} when $L$ is the WW-hinge loss and $\ell$ is the ordered partition loss:
  \begin{definition}
    \label{main-definition: embedding map varphi}
    The \emph{embedding} $\varphi : \mathcal{OP}_k \to \mathbb{R}^{k}$ is defined as follows:
  Let $\mathbf{S} = (S_1,\dots, S_l) \in \mathcal{OP}_k$.
    Define $\varphi(\mathbf{S}) \in \mathbb{R}^{k}$ entrywise so that for all $i \in [l_{\mathbf{S}}]$ and all $j \in S_i$, we have
  $
    [\varphi(\mathbf{S})]_{j} = -(i-1).
  $
  \end{definition}
  With the discrete loss $\ell$ and the embedding map $\varphi$ defined, we now proceed to the main results.

\section{Main results}
\label{main-section: main results}
In this work, we establish that the WW-hinge loss embeds the ordered partition loss:
\begin{theorem}
  \label{main-theorem: embedding result only}
  The Weston-Watkins hinge loss $L: \mathbb{R}^k \to \mathbb{R}^k$ embeds the ordered partition loss $\ell: \mathcal{OP}_k \to \mathbb{R}^k$ with embedding $\varphi$ as in \cref{main-definition: embedding map varphi}.
\end{theorem}

In light of \cref{main-theorem: embedding implies calibration}, \cref{main-theorem: embedding result only} implies
\begin{corollary}
  \label{main-corollary: calibration result only}
  $L$ is calibrated with respect to $\ell$.
\end{corollary}

In the remainder of this section, we develop the tools necessary to prove \cref{main-theorem: embedding result only}.

\subsection{Vectorial representation of ordered partitions}
First, we define the set $\SkCZ$ whose elements serve as realizations of ordered partitions inside $\mathbb{R}^k$.
\begin{definition}
  \label{main-definition: SkCZ}
  Define the following sets:
  \begin{equation}
    \mathcal{C}
    :=
    \{
      v \in \mathbb{R}^{k} : v_1 = 0, v_k \le -1,\,v_i - v_{i+1} \in [0,1], \, \forall i \in [k-1]
    \},
    \qquad
    \mathcal{C}_{\mathbb{Z}} := \mathcal{C} \cap \mathbb{Z}^k
  \end{equation}
  and finally
  $
    \SkCZ := \bigcup_{\sigma \in \Perm{k}} \sigma \mathcal{C}_{\mathbb Z}
  $
  where $\sigma \mathcal{C}_{\mathbb Z} = \{\sigma v : v \in \mathcal{C}_{\mathbb Z}\}$.
\end{definition}

A vector $v \in \mathbb{R}^k$ is \emph{monotonic non-increasing} if $v_1 \ge v_2 \ge \cdots \ge v_k$.
Note that vectors in $\mathcal{C}_{\mathbb{Z}}$ are nonconstant, integer-valued monotonic non-increasing such that consecutive entries decrease at most by $1$.
Furthermore, by construction, $\SkCZ$ consists of all possible permutations of elements in $\mathcal{C}_{\mathbb{Z}}$.
Therefore, the entries of an element $v \in \SkCZ$ take on every value in $0,-1,\dots, -(l-1)$ for some integer $l \in \{2,\ldots,k\}$.
Thus, $v \in \SkCZ$ can be thought of as vectorial representation of the ordered partition $\mathbf{S} = (S_1,\dots, S_l)$ where $S_i = \{j : v_j = -(i-1)\}$ for each $i \in [l]$.
In \cref{main-theorem: bijection between OP representations} below, we make this notion precise.

  \begin{lemma}
    \label{main-lemma: image of the embedding map is SkCZ}
    The image of $\varphi$ is contained in $\SkCZ$.  
  \end{lemma}
  \begin{proof}
    Let $\mathbf{S} \in \mathcal{OP}_k$. It suffices to prove that there exists some $\sigma \in \Perm{k}$ such that $\sigma \varphi(\mathbf{S}) \in \mathcal{C}_{\mathbb{Z}}$.
    Note that by definition, we have the set of unique values of $\varphi(\mathbf{S})$ is
    \[
      \{[\varphi(\mathbf{S})]_j : j \in [k]
      \}
      =
      \{0, -1, -2,\dots, -(l_{\mathbf{S}}-1)\}.
    \]
    Thus, let $\sigma \in \Perm{k}$ be such that $\sigma \varphi(\mathbf{S})$ is monotonic non-increasing. Then $\sigma \varphi(\mathbf{S}) \in \mathcal{C}_{\mathbb{Z}}$.
  \end{proof}

  Next, we define the inverse of $\varphi$.

  \begin{definition}
    \label{main-definition: quasi-link map}
    The \emph{quasi-link map} $\tilde \psi : \SkCZ \to \mathcal{OP}_k$ is defined as follows:
    Given $v \in \SkCZ$, let $l = 1-\min_{j\in[k]} v_j$.
    Define $ S_i = \{j \in [k]: v_j = -(i-1)\}$
for each $i \in [l]$.
    Finally, define $\tilde \psi(v)
    =(S_1,\dots, S_l)$.
  \end{definition}
  The tilde in $\tilde \psi$ is to differentiate the quasi-link from $\psi$ in \cref{main-definition: calibration}.
\begin{proposition}
  \label{main-theorem: bijection between OP representations}
  The embedding map $\varphi: \mathcal{OP}_k \to \SkCZ$ given in \cref{main-definition: embedding map varphi} is a bijection with inverse given by the quasi-link map $\tilde \psi$ from \cref{main-definition: quasi-link map}.
\end{proposition}
\begin{proof}
  
  We first show that for all $\tilde \psi(\varphi(\mathbf{S})) = \mathbf{S}$ for all $\mathbf{S} = (S_1,\dots, S_l)  \in \mathcal{OP}_k$.
  Observe that $S_i = \{j \in[k] : [\varphi(\mathbf{S})]_j = -(i-1)\}$ for all $i = 1,2,\dots, l$. This implies that $\tilde \psi(\varphi(\mathbf{S})) = \mathbf{S}$.

  Next, we show that $\varphi(\tilde \psi(v)) = v$ for all $v \in \SkCZ$.
  Let $\mathbf{S} = (S_1,\dots, S_l) = \tilde \psi(v)$.
  Then
  $[\varphi(\mathbf{S})]_j
  =-(i-1)$ if and only if $j \in S_i$. By definition $S_i = \{j \in [k]: v_j = -(i-1)\}$. Hence,
  $[\varphi(\mathbf{S})]_j
  =-(i-1)$ if and only if $v_j = -(i-1)$ which implies that $\varphi(\mathbf{S}) = v$, as desired.
\end{proof}
In the next section, using $\varphi$, we prove a relationship between the inner risk functions of $L$ and $\ell$.
\subsection{Inner risk functions}
Define the \emph{inner risk} functions $\underline L: \Delta^k \to \mathbb{R}_+$ and $\underline{\ell} : \Delta^k \to \mathbb{R}_+$ as follows:
\begin{equation}
  \label{main-equation: WW hinge bayes risk}
  \underline L(p) = \inf_{v \in \mathbb{R}^{k}}\langle p, L(v) \rangle,
  \quad \mbox{and} \quad
  \underline \ell(p) = \inf_{\mathbf{S} \in \mathcal{OP}_k}\langle p, \ell(\mathbf{S}) \rangle.
\end{equation}
Note that these functions appear in the second part of \cref{main-definition: loss embedding}, although here we have $\inf$ instead of $\min$.
Since $\mathcal{OP}_k$ is finite, the infimum in 
the definition of $\underline{\ell}$ is attained.
Later, we will argue that the infimum in the definition of $\underline{L}$ is also attained.

We now state the main structural result regarding $\underline{L}$:
  \begin{theorem}
    \label{main-theorem: L bayes risk is optimized over OP}
    For all $p \in \Delta^k$, we have
    \[
      \underline L(p) = \min_{v \in \SkCZ} \langle p, L(v) \rangle.
    \]
  \end{theorem}
  \begin{proof}[Sketch of proof]
    Note that $L$ is invariant under translation by any scalar multiple of the all ones vector. Thus, $L$ has an extra degree of freedom. We introduce a loss function $\LL: \mathbb{R}^{k-1} \to \mathbb{R}^k$ called the \emph{reduced WW-hinge loss}, which removes this extra degree freedom.
    Furthermore, there exists a mapping $\pi : \mathbb{R}^k \to \mathbb{R}^{k-1}$ such that $\langle p, L(v) \rangle
    = 
    \langle p, \LL(\pi (v))\rangle$ for all $p \in \Delta^k$ and $v \in \mathbb{R}^k$.
    Letting $z = \pi(v) \in \mathbb{R}^{k-1}$, we show that for a fixed $p$, the function $F_p(z) := \langle p ,\LL(z)\rangle$ is convex and piecewise-linear and the minimization of which can be formulated as a linear program \cite{bertsimas1997introduction}.
    Furthermore, since $F_p$ is nonnegative, the infimum $\inf_{z \in \mathbb{R}^{k-1}} F_p(z)$ is attained \cite[Corollary 3.2]{bertsimas1997introduction}, which implies that the infimum in the definition of $\underline{L}$ in \cref{main-equation: WW hinge bayes risk} is attained as well.
    The linear program is shown to be totally unimodular, which implies that an integral solution exists \cite{lawler2001combinatorial}, i.e., $\min_{z \in \mathbb{R}^{k-1}} F_p(z) = F_p(z^*)$ for some $z^* \in \mathbb{Z}^{k-1}$.
    From $z^*$, we obtain an integral $v^* \in \mathbb{Z}^k$ such that $\underline L(p) = \langle p, L(v^*)\rangle$.
    Finally, we construct an element $v^\dagger \in \SkCZ$ from $v^*$ in such a way that the objective does not increase, i.e.,
    $\langle p, L(v^*)\rangle \ge \langle p, L(v^\dagger)\rangle$, which implies that $\underline L(p) = \langle p, L(v^\dagger)$ by the optimality of $v^*$.
  \end{proof}

The ordered partition loss $\ell$ and the WW-hinge loss $L$ are related by the following:
\begin{theorem}
  \label{main-theorem: L and ell inner risk identity}
  For all $p \in \Delta^k$ and all $\mathbf{S} \in \mathcal{OP}_k$, we have
  \[
  \langle p, \ell(\mathbf{S}) \rangle
  =
\langle p, L(\varphi(\mathbf{S}))\rangle,\]
where $\varphi$ is the embedding map as in \cref{main-definition: embedding map varphi}.
\end{theorem}
\begin{proof}
  [Sketch of proof]
  Let $\mathbf{S} = (S_1,\dots, S_l) \in \mathcal{OP}_k$ and $p \in \Delta^k$.
  Let $T \in \mathbb{R}^{k\times k}$ consist of ones on and below the main diagonal and zero everywhere else. 
Letting $D = T^{-1}$, we have \[\langle p, L(\varphi(\mathbf{S}))\rangle
  =
  \langle p, TD L(\varphi(\mathbf{S}))\rangle
  =
  \langle T'p, D L(\varphi(\mathbf{S}))\rangle
  .\]
  Next, we observe that $[T'p]_i = p_i + \cdots + p_k$ for each $i \in [k]$.
  We then show through a lengthy calculation that for each $i \in [k]$
   \begin{enumerate}[wide, labelwidth=!, labelindent=0pt]
     \item If $i = 1$, then 
    $[T'p]_1 = 1$ and $[DL(\varphi(\mathbf{S}))]_1 = |S_1| - 1$.
  \item If $i > 1$ and $i = |S_1 \cup \cdots \cup S_j|+1$ for some $j \in [l]$, then
    $
    [T'p]_i = \Pr_{Y \sim p} \left\{ Y \not \in S_1 \cup \dots \cup S_j
  \right\} 
  $ and $
  [DL(\varphi(\mathbf{S}))]_i
  =
  |S_1\cup \cdots \cup S_{j+1}|.
  $
\item For all other $i$, $ [DL(\varphi(\mathbf{S}))]_i =0$ (in which case the value of $[T'p]_i$ is irrelevant).
   \end{enumerate}
  From this, we deduce that $
  \langle T'p, D L(\varphi(\mathbf{S}))\rangle$ is equal to \cref{main-equation: ordered partition inner risk}.
\end{proof}

Next, we show that the inner risks of $L$ 
  and $\ell$ from
  \cref{main-equation: WW hinge bayes risk}
  are in fact identical:

\begin{corollary}
  \label{main-lemma: L and ell bayes risk are equal}
For all $p \in \Delta^k$, we have
$\underline{L}(p) = \underline{\ell}(p)$.
\end{corollary}
\begin{proof}
  Observe that
  \[
    \underline{\ell}(p)
    \overset{(a)}{=}
    \min_{\mathbf{S} \in \mathcal{OP}_k} \langle p, \ell(\mathbf{S})\rangle
    \overset{(b)}{=}
    \min_{\mathbf{S} \in \mathcal{OP}_k} 
    \langle p, L(\varphi(\mathbf{S}))\rangle
    \overset{(c)}{=}
    \min_{v \in \SkCZ} 
    \langle p, L(v)\rangle
    \overset{(d)}{=}
    \underline{L}(p)
  \]
  where $(a)$ follows from definition of $\underline{\ell}$,
  $(b)$ from \cref{main-theorem: L and ell inner risk identity},
  $(c)$ from the fact that $\varphi : \mathcal{OP}_k \to \SkCZ$ is a bijection (\cref{main-theorem: bijection between OP representations}),
  and $(d)$ from \cref{main-theorem: L bayes risk is optimized over OP}.
\end{proof}
Having developed all the tools necessary, we turn toward the proof of our main result \cref{main-theorem: embedding result only}.

\subsection{Proof of \cref{main-theorem: embedding result only}}
  We check that the two conditions in \cref{main-definition: loss embedding} holds. 
  The first condition is that $L(\varphi(\mathbf{S})) = \ell(\mathbf{S})$ for all $\mathbf{S} \in \mathcal{OP}_k$, which follows from \cref{main-theorem: L and ell inner risk identity}.
  To see this, note that for all $i \in [k]$ the $i$-th elementary basis vector $e_i \in \Delta^k$. Thus, we have
  \[
  [L(\varphi(\mathbf{S}))]_i
  =
    \langle e_i, L(\varphi(\mathbf{S}))\rangle
  =
\langle e_i, \ell(\mathbf{S})\rangle
=
  [\ell(\mathbf{S})]_i
\] for all $i\in[k]$. This implies that $L(\varphi(\mathbf{S})) = \ell(\mathbf{S})$, which is the first condition of \cref{main-definition: loss embedding}.

  Next, we check the second condition.
  Let $p \in \Delta^k$.
Define
$
  \gamma(p)
  :=
  \argmin_{\mathbf{S} \in \mathcal{OP}_k} \langle p, \ell(\mathbf{S}) \rangle,
  $ and $
  \Gamma(p)
  :=
  \argmin_{
  v \in \mathbb{R}^{k}}
  \langle p,
  L(v)
  \rangle.
$
    Furthermore, by the definition of $\gamma$, 
    $\mathbf{S} \in \gamma(p)$
    if and only if
    $\langle p, \ell(\mathbf{S}) \rangle = \underline{\ell}(p)$.
    Likewise, $\varphi(\mathbf{S}) \in \Gamma(p)$ 
    if and only if
    $\langle p, L(\varphi(\mathbf{S})) \rangle = \underline{L}(p)$.
    By \cref{main-lemma: L and ell bayes risk are equal} and 
    \cref{main-theorem: L and ell inner risk identity}, 
    we have
    $\langle p, \ell(\mathbf{S}) \rangle 
    = \underline{\ell}(p)$ 
    if and only if
    $\langle p, L(\varphi(\mathbf{S})) \rangle 
    = \underline{L}(p)$.
Putting it all together, we get
  $
    \mathbf{S} \in \gamma(p)
    $
    if and only if
    $
    \varphi(\mathbf{S}) \in \Gamma(p)
    $, which is the second condition of \cref{main-definition: loss embedding}.

    \section{Maximally informative losses}
    \label{main-section: maximally informative losses}
Going forward, let $L: \mathbb{R}^d \to \mathbb{R}^k_+$ be a generic surrogate loss.
The WW-hinge loss is denoted by $L^{WW}$ and the CS-hinge loss by $L^{CS}$.
Likewise, let $\ell : \mathcal{R} \to \mathbb{R}^k_+$ be a generic discrete loss.
The ordered partition loss is denoted by $\ell^{\mathcal{OP}}$ and the 0-1 loss by $\ell^{zo}$.

    We define the ``dual'' notion to the convex calibration dimension \cite{ramaswamy2016convex}:
\begin{definition}
  Let $L: \mathbb{R}^d \to \mathbb{R}^k_+$ be a loss. Define the \emph{embedding cardinality} of $L$ as
  \[
    \mathrm{emb.card}(L)
    :=
    \min
    \left\{ n \in \{2,3,\dots\}
      \mid
      \substack{
      \mbox{there exists a discrete loss } \ell: [n] \to \mathbb{R}^k
      \\
    \mbox{such that $L$ embeds $\ell$}
  }
\right\}.
  \]
  A discrete loss $\ell: \mathcal{R} \to \mathbb{R}^k$ is said to be \emph{maximally informative} for $L$ if $|\mathcal{R}| = \mathrm{emb.card}(L)$ and $L$ embeds $\ell$.
\end{definition}

  For each $k\in \{3,\dots, 7\}$,
     we showed that by a computer search that
for all $\mathbf{S} \in \mathcal{OP}_k$,
there exists $p \in \Delta^k$ such that $\mathbf{S}$ is the \emph{unique} minimizer of $\min_{\mathbf{T} \in \mathcal{OP}_k} \langle p, \ell(\mathbf{T})\rangle$.
A consequence of this is that
\begin{proposition}
  \label{main-proposition: maximal informative}
  For $k\in \{3,\dots, 7\}$, 
  $\mathrm{emb.card}(L^{WW})
  =|\mathcal{OP}_k|$. In other words, the ordered partition loss is maximally informative for the WW-hinge loss.
\end{proposition}

\section{The argmax link}
\label{main-section: the argmax link}

Define
$
  \gamma_\ell(p)
  :=
  \argmin_{r \in \mathcal{R}}
  \langle p , \ell(r)\rangle$
   and
  $\Gamma_L(p)
  :=
  \argmin_{v \in \mathbb{R}^d}
  \langle p, 
  L(v)
  \rangle.
$
For multiclass classification into $k$ classes, most multiclass SVMs typically output a vector of scores $v \in \mathbb{R}^k$ which is converted to a class label by taking $\argmax v$.
In this section, we analyze the $\argmax$ as a ``link'' function.
Recall from \cref{main-section: notations}, $\argmax$ is a set-valued function.
Define
\[
  \Omega_L
  :=
  \{ p \in \Delta^k : 
    |\argmax p| = 1,\,
  \argmax v = \argmax p,\, \forall v \in \Gamma_L(p)\}.
\]
When $L$ is calibrated with respect to $\ell^{zo}$, we have that $\Omega_L = \{p \in \Delta^k: |\argmax p| = 1\}$. Hence, $\Delta^k \setminus \Omega_L$ has measure zero.
For other $L$ not necessarily calibrated with respect to $\ell^{zo}$, it is desirable that $\Omega_L$ be as large as possible.
Below, we will prove that $\Omega_{L^{CS}}$ is a proper subset of $\Omega_{L^{WW}}$.

Recall that $\mathcal{X}$ is a sample space and $P$ is a distribution on $\mathcal{X} \times [k]$.
For each $x \in \mathcal{X}$, define the \emph{class conditional distribution} $\eta_P(x) \in \Delta^k$ by
$[\eta_P(x)]_y = \Pr_{X,Y \sim P} (Y = y | X=x)$.
  \begin{proposition}
    \label{main-theorem: 0-1 loss Bayes optimality}
    Let $P$ be a joint distribution on $\mathcal{X} \times [k]$ such that $\eta_P(x) \in \Omega_L$ for all $x$ and $L : \mathbb{R}^d \to \mathbb{R}^k_+$ be a loss.
    Let $g^* : \mathcal{X} \to \mathbb{R}^k$ be such that $g^*(x) \in \Gamma_L(\eta_P(x))$ for all $x \in \mathcal{X}$.
    Then $\argmax \circ g^*$ is Bayes optimal with respect to the 0-1 loss.
  \end{proposition}
  \begin{proof}
    By definition of $\Omega_L$, we have $\argmax \circ g^*(x) = \argmax \eta_P(x)$ for all $x \in \mathcal{X}$.
  \end{proof}

  The following theorem asserts that for any $v \in \Gamma_{L^{WW}}(p)$, the $\argmax v$ is contained in the top bucket $S_1$ for some $\mathbf{S} \in \gamma_{\ell^{\mathcal{OP}}}(p)$.
\begin{theorem}
  \label{main-theorem: argmax v is contained in S1}
  Let $p \in \Delta^k$ be such that $\max p > \frac{1}{k}$ and $v \in \Gamma_{L^{WW}}(p)$. Then there exists $\mathbf{S} = (S_1,\dots, S_l) \in \gamma_{\ell^{\mathcal{OP}}}(p)$ such that
  $\argmax v \subseteq S_1$.
\end{theorem}
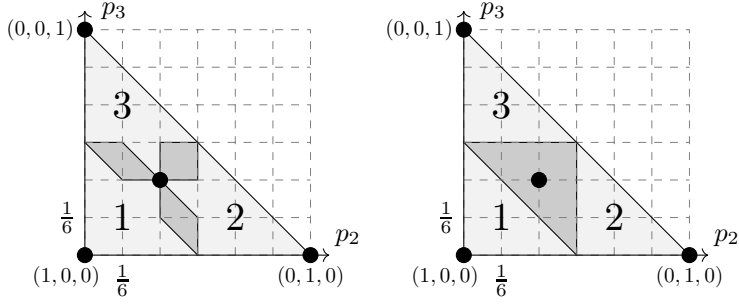
\begin{figure}
  \centering
  \begin{tikzpicture}[scale = 0.5]
    \fill[black!5!white] (0,0) -- (6,0) -- (0,6) -- cycle;
    \fill[black!20!white] (2,2) -- (2,1) -- (3,0) -- (3,1) -- cycle;
    \fill[black!20!white] (2,2) -- (3,2) -- (3,3) -- (2,3) -- cycle;
    \fill[black!20!white] (2,2) -- (1,3) -- (0,3) -- (1,2) -- cycle;
    \draw (2,2) -- (2,1) -- (3,0) -- (3,1) -- cycle;
    \draw (2,2) -- (3,2) -- (3,3) -- (2,3) -- cycle;
    \draw (2,2) -- (1,3) -- (0,3) -- (1,2) -- cycle;
    \draw[help lines, dashed] (0,0) grid (6,6);
    \draw (0,6) -- (6,0);
    \node[draw,circle,inner sep=2pt,fill] at (0,0) {};
    \node[below, scale = 0.8] at (-0.5,-0.1) {$(1,0,0)$};
    \node[draw,circle,inner sep=2pt,fill] at (0,6) {};
    \node[left, scale = 0.8] at (-0.1,6) {$(0,0,1)$};
    \node[draw,circle,inner sep=2pt,fill] at (2,2) {};
    \node[draw,circle,inner sep=2pt,fill] at (6,0) {};
    \node[below, scale = 0.8] at (6,-.1) {$(0,1,0)$};
    \node[below] at (1,0) {$\frac{1}{6}$};
    \node[above] at (7,0) {$p_2$};
    \node[right] at (0.2,6.5) {$p_3$};
    \draw [<->] (0,6.5) -- (0,0) -- (6.5,0);
    \node[left] at (0,1) {$\frac{1}{6}$};
    \node[scale = 1.5] at (1,1) {1};
    \node[scale = 1.5] at (4,1) {2};
    \node[scale = 1.5] at (1,4) {3};
  \end{tikzpicture} 
  \begin{tikzpicture}[scale = 0.5]
    \fill[black!5!white] (0,0) -- (6,0) -- (0,6) -- cycle;
    \fill[black!20!white] (3,0) -- (3,3) -- (0,3) -- cycle;
    \draw (3,0) -- (3,3) -- (0,3) -- cycle;
    \draw[help lines, dashed] (0,0) grid (6,6);
    \draw (0,6) -- (6,0);
    \node[draw,circle,inner sep=2pt,fill] at (0,0) {};
    \node[below, scale = 0.8] at (-0.5,-0.1) {$(1,0,0)$};
    \node[draw,circle,inner sep=2pt,fill] at (0,6) {};
    \node[left, scale = 0.8] at (-0.1,6) {$(0,0,1)$};
    \node[draw,circle,inner sep=2pt,fill] at (2,2) {};
    \node[draw,circle,inner sep=2pt,fill] at (6,0) {};
    \node[below, scale = 0.8] at (6,-.1) {$(0,1,0)$};
    \node[below] at (1,0) {$\frac{1}{6}$};
    \node[above] at (7,0) {$p_2$};
    \node[right] at (0.2,6.5) {$p_3$};
    \draw [<->] (0,6.5) -- (0,0) -- (6.5,0);
    \node[left] at (0,1) {$\frac{1}{6}$};
    \node[scale = 1.5] at (1,1) {1};
    \node[scale = 1.5] at (4,1) {2};
    \node[scale = 1.5] at (1,4) {3};
  \end{tikzpicture} 
  \caption{
    The gray triangle represents the probability simplex $\Delta^3$, where $(p_1,p_2,p_3) \in \Delta^3$ is plotted as $(p_2,p_3)$ in the plane.
    The light gray regions are  $\Omega_{L^{WW}}$ (left) and
     $\Omega_{L^{CS}}$ (right).
   For the derivation, see \cref{section: figure 2 detailed derivation}.}
  \label{main-figure: bayes decision regions of WW vs CS}
\end{figure}
Below, we consider two conditions on $p \in \Delta^k$ such that for \emph{all} $\mathbf{S}\in \gamma_{\ell^{\mathcal{OP}}}(p)$, the top bucket $S_1 = \argmax p$.
By \cref{main-theorem: argmax v is contained in S1}, for such $p \in \Delta^k$, we can recover $\argmax p$ from any $v \in \Gamma_{L^{WW}}(p)$.
The first condition covers $p \in \Delta^k$ such that the top class has a majority:
\begin{proposition}
  \label{main-proposition: majority case}
  Let $p \in \Delta^k$ satisfy the ``majority condition'': $\max p > 1/2$.
  Then for all $\mathbf{S} = (S_1,\dots, S_l) \in \gamma_{\ell^{\mathcal{OP}}}(p)$, we have $|S_1| = 1$ and $S_1 = \argmax p$.
\end{proposition}
While \cref{main-proposition: majority case} does not guarantee that $\gamma_{\ell^{\mathcal{OP}}}(p)$ is a singleton, all $\mathbf{S} \in \gamma_{\ell^{\mathcal{OP}}}(p)$ have the same top bucket.
The second condition covers $p \in \Delta^k$ whose top class may not have a majority, yet $\argmax p$ can still be recovered from any $v \in \Gamma_{L^{WW}}(v)$ by taking $\argmax v$:
\begin{proposition}
  \label{main-proposition: spiked case}
  Fix a number $\alpha$ such that $1> \alpha > \frac{1}{k}$.
  Let $p \in \Delta^k$ satisfy the ``symmetric label noise (SLN) condition'':
    there exists $j^* \in [k]$ so that
    $p_{j^*} = \alpha$ and $p_j = \frac{1-\alpha}{k-1}$ for all $j \ne j^*$.
  Then 
  $
      (\{j^*\},\, [k] \setminus \{j^*\})
  $
  is the unique element of $\gamma_{\ell^{\mathcal{OP}}}(p)$.
\end{proposition}
In particular, when $\alpha <1/2$, $p$ violates the majority condition.
Under SLN, we have $\argmax p = \{j^*\}$ since
$
  \alpha - \frac{1-\alpha}{k-1}
  =
  \frac{(k-1)\alpha - 1 + \alpha}{k-1}
  =
  \frac{k\alpha - 1}{k-1}
  >
  \frac{1- 1}{k-1}
  =0.
$
In light of 
  \cref{main-theorem: argmax v is contained in S1}, we have 
  \begin{corollary}
    \label{main-corollary: Omega for the WW-hinge loss}
    If $p \in \Delta^k$ satisfies the majority or the SLN condition, then $p \in \Omega_{L^{WW}}$.
  \end{corollary}
  This supports the observation by \citet{dogan2016unified} that the WW-SVM performs well even without the majority condition.
  For the CS-hinge loss, it is known that $\Omega_{L^{CS}} = \{p \in \Delta^k : p \mbox{ satisfies the majority condition}\}$ \cite[Lemma 4]{liu2007fisher}.
  In particular, $\Omega_{L^{CS}}$ is a proper subset of $\Omega_{L^{WW}}$.
  For $k=3$, we show in \cref{main-figure: bayes decision regions of WW vs CS}
  the regions $\Omega_{L^{WW}}$ and $\Omega_{L^{CS}}$.

  \section{Conclusion and future work}
  We proved that the Weston-Watkins hinge loss is calibrated with respect to the ordered partition loss, which we argue is maximally informative for the WW-hinge loss.
  Furthermore, we showed the advantage of WW-hinge loss over the Crammer-Singer hinge loss when the popular ``argmax'' link is used.
  An interesting direction is to apply the ordered partition loss to other multiclass learning problems such as partial label and multilabel learning.

  \newpage

\bibliographystyle{unsrtnat}
\bibliography{references}

\newpage

\appendixpage
\appendix
\section{Organization of contents}
In 
\cref{section: notations and conventions},
we introduce notations in addition to those already defined in in the main article's \cref{main-section: notations}.

In \cref{section: main results,section: maximally informative losses,section: the argmax link}
, we present the proofs and supporting theory for all results from 
\cref{main-section: main results,main-section: maximally informative losses,main-section: the argmax link}, respectively.

In \cref{section: derivation of the figures}, we discuss how \cref{main-figure: Bayes optimal classifier of the OP loss,main-figure: bayes decision regions of WW vs CS} are obtained.

\section{Additional notations}
\label{section: notations and conventions}
\begin{itemize}
  \item Below, $L$ always denotes the WW-hinge loss (\cref{main-definition: full WW hinge loss}) and $\ell$ always denotes the ordered partition loss (\cref{main-definition: ordered partition loss}).
  \item All vectors are column vectors unless stated otherwise.
  \item $\mathbb{R}_+$ and $\mathbb{Z}_+$ denotes the set of non-negative reals and integers, respectively.
  \item Define $\mathbb{R}^k_\uparrow 
    =\{v \in \mathbb{R}^{k}: v_1 \le v_2 \le \cdots \le v_k\}.$ Likewise, define $\mathbb{R}^k_\downarrow$.
  \item For a positive integer $n$, we let $[n] := \{1,\dots, n\}$. By convention, $[0] = \emptyset$.
  \item Let $\mathbf{1}^k \in \mathbb{R}^{k}$ denote the vector all ones.
  \item 
For a number $t \in \mathbb{R}$, let $[t]_+ = \max\{0,t\}$. For a vector $v$, we denote by $[v]_+$ the vector resulting from applying $[\cdot]_+$ entrywise to $v$. The \emph{hinge loss} $h: \mathbb{R} \to \mathbb{R}_+$ is defined by $h(x) = [1-x]_+$.
\item For a vector $v \in \mathbb{R}^k$, we use $[v]_i$ to denote the $i$-th entry of $v$ in conjunction with the usual notation $v_i$.

\item 
Given a vector $v \in \mathbb{R}^k$, we define
\[
  \max v := \max_{i \in [k]} v_i
  \quad \mbox{and} \quad
  \argmax v := \{i \in [k]: v_i = \max v\}
\]
Define $\min v$ and $\argmin v$ likewise.
  \item Probability simplex
    \[
      \Delta^k = \{p = (p_1,\dots, p_k) \in \mathbb{R}^{k}_{+} : p_1 + \cdots + p_k = 1\}
    \]
    and \emph{non-increasing} probability simplex
    \[
      \Delta^k_{\downarrow}
      =
      \{p \in \Delta^k : p_1 \ge p_2 \ge \cdots \ge p_k\}
      =
      \Delta^k \cap \mathbb{R}^k_\downarrow.
    \]
  \item For $p \in \Delta^k$, we write $Y\sim p$ to denote a discrete random variable $Y \in [k]$ whose probability mass function is $p$.
  \item For each $i,j \in [k]$, $\sigma_{(i,j)} \in \mathbb{R}^{k\times k}$ is the permutation matrix that switches the $i$-th and $j$-th index.
    By convention, if $i = j$, then $\sigma_{(i,j)}$ is the identity. Also, for brevity, define $\sigma_i = \sigma_{(1,i)}$.
  \item According to the definition above, $\sigma_{(i,j)}$ acts on $\mathbb{R}^k$. 
    However, we abuse notation and allow $\sigma_{(i,j)}$ to act on $[k]$ in the obvious way.
    In such cases, we write $\sigma_{(i,j)}(\ell)$ for $\ell \in [k]$.
\end{itemize}

\section{Main results}
\label{section: main results}

  \begin{lemma}
    \label{lemma: L is translation invariant}
    For all $v \in \mathbb{R}^{k}$ and $c \in \mathbb{R}$, we have $L(v) = L(v + c\mathbf{1}^k)$.
  \end{lemma}
  \begin{proof}
    For all $y \in [k]$, we have that
\[
  [L(v + c\mathbf{1})]_y = 
  \sum_{i \in [k]\,:\, i \ne y} 
  h(v_y + c - (v_i-c))=
  \sum_{i \in [k]\,:\, i \ne y} 
  h(v_y  - v_i)
  =
  [L(v)]_y.
\]
  \end{proof}

  \begin{lemma}
    \label{lemma: L is transposition equivariant}
    For all $j \in [k]$,
    we have
    $L(\sigma_j v) = \sigma_j L(v)$.
  \end{lemma}
  \begin{proof}
    If $j = 1$, then the result is trivial.
    Hence, let $j > 1$.
    We  prove
    \begin{equation*}
      [L(\sigma_j v)]_y = [L(v)]_{\sigma_j(y)}
    \end{equation*} 
    for the following three cases: $y \not \in \{1,j\}$, $y = 1$ and $y = j$.
    Before we go through the cases, note that
    \[
      [L(\sigma_j v)]_y
      =
      \sum_{i \in [k] : i \ne y}
      h([\sigma_j v]_y - [\sigma_j v]_i)
      =
      \sum_{i \in [k] : i \ne y}
      h(v_{\sigma_j(y)} - v_{\sigma_j(i)}).
    \]
    Now, for the first case, suppose that $y \not \in \{1,j\}$. Then $\sigma_j(y) = y$ and so
    \begin{align*}
      [L(\sigma_j v)]_y
      &=
      \sum_{i \in [k] : i \ne y}
      h(v_{y} - v_{\sigma_j(i)})
      \\
      &=
    h(v_y - v_{\sigma_j(1)})
      +
    h(v_y - v_{\sigma_j(j)})
      +
      \sum_{i \in [k] \, :\, i \not \in \{1,j,y\}}
      h(v_{y} - v_{\sigma_j(i)})
      \\
      &=
    h(v_y - v_{j})
      +
    h(v_y - v_{1})
      +
      \sum_{i \in [k] \, :\, i \not \in \{1,j,y\}}
      h(v_{y} - v_{i})
      \\
      &=
      \sum_{i \in [k] \, :\, i \not \in \{y\}}
      h(v_{y} - v_{i})
      \\
      &= 
      [L(v)]_y
      =
      [L(v)]_{\sigma_j(y)}.
    \end{align*}
    Next, suppose that $y = 1$. Thus, we have $\sigma_j(y) = \sigma_j(1) = j$. So
    \begin{align*}
      [L(\sigma_j v)]_y
      =
      [L(\sigma_j v)]_1
      &=
      \sum_{i \in [k] : i \ne 1}
      h(v_{j} - v_{\sigma_j(i)})
      \\
      &=
      \sum_{i \in [k] : i \ne j}
      h(v_{j} - v_{i})
      \\
      &=
      [L(v)]_j = [L(v)]_{\sigma_j(y)}.
    \end{align*}
    Finally, if $y = j$, $\sigma_j(y) = 1$
    \begin{align*}
      [L(\sigma_j v)]_y
      =
      [L(\sigma_j v)]_j
      &=
      \sum_{i \in [k] : i \ne j}
      h(v_{1} - v_{\sigma_j(i)})
      \\
      &=
      \sum_{i \in [k] : i \ne 1}
      h(v_{j} - v_{i})
      \\
      &=
      [L(v)]_1 = [L(v)]_{\sigma_j(j)}=[L(v)]_{\sigma_j(y)}.
    \end{align*}
    
  \end{proof}

\begin{lemma}
  \label{lemma: transposition triple identity}
  Let $i,j \in \{2,\dots, k\}$ be distinct. Then
  $\sigma_{i} \sigma_{j} \sigma_{i} = \sigma_{(i,j)}$.
\end{lemma}
\begin{proof}
  This is simply an exhaustive case-by-case proof over all inputs $y \in [k]$.
  First, let $y=1$. Then $\sigma_{(i,j)}(1) = 1$ since $1 \not \in \{i,j\}$. On the other hand
  $
    \sigma_{i} \sigma_{j} \sigma_{i} (1)
    = 
    \sigma_{i} \sigma_{j} (i)
    =
    \sigma_{i} (i)
    =1.
  $
  Now, let $y \in \{2,\dots, k\}$.  If $y \not \in \{i,j\}$, then $\sigma_{(i,j)}(y) = y$ and
  $
    \sigma_{i} \sigma_{j} \sigma_{i} (y)
    = 
    \sigma_{i} \sigma_{j} (y)
    =
    \sigma_{i} (y)
    =y.
  $
  If $y = i$, then $\sigma_{(i,j)}(i) = j$ and
  $
    \sigma_{i} \sigma_{j} \sigma_{i} (i)
    = 
    \sigma_{i} \sigma_{j} (1)
    =
    \sigma_{i} (j)
    =j.
  $
  If $y = j$, then $\sigma_{(i,j)}(j) = i$ and
  $
    \sigma_{i} \sigma_{j} \sigma_{i} (j)
    = 
    \sigma_{i} \sigma_{j} (j)
    =
    \sigma_{i} (1)
    =i.
  $
\end{proof}

\begin{corollary}
  \label{corollary: symmetric group is generated by transpositions}
  Every $\sigma \in \Perm{k}$ can be written as a product $\sigma = \sigma_{i_1} \sigma_{i_2} \cdots  \sigma_{i_l}$.
\end{corollary}
\begin{proof}
  We prove the equivalent statement that the set $\mathcal{S} := \{\sigma_i: i \in \{2,\dots, k\}\}$ generates the group $\Perm{k}$.
  A standard result in group theory states that the set of transpositions $\mathcal{T}$ generates $\Perm{k}$.
  By \cref{lemma: transposition triple identity}, transpositions between labels in $\{2,\dots, k\}$ can be generated by $\mathcal{S}$.
  Furthermore, $\sigma_i = \sigma_{(1,i)}$ by definition, so transposition between $1$ and elements of $\{2,\dots, k\}$ can be generated by $\mathcal{S}$ as well.
  Hence, all of $\mathcal{T}$ can be generated by $\mathcal{S}$.
\end{proof}

\begin{corollary}
  \label{corollary: L is Sigma equivariant}
  For all $v \in \mathbb{R}^{k}$ and $\sigma \in \Perm{k}$, we have
  \[
    L(\sigma v) = \sigma L(v).
  \]
\end{corollary}
\begin{proof}
  By \cref{corollary: symmetric group is generated by transpositions}, we may write $\sigma = \sigma_{i_1} \sigma_{i_2} \cdots \sigma_{i_m}$.  Hence,
  \begin{align}
     L(\sigma v) 
    &=
     L(\sigma_{i_1} \sigma_{i_2} \cdots \sigma_{i_m}v) 
    \\ &=
    \sigma_{i_1}  L( \sigma_{i_2} \cdots \sigma_{i_m}v) 
    \label{equation: Sigma-equivariance 3}
    \\ &\,\,\,\vdots \nonumber
    \\ &=
    \sigma_{i_1}  \sigma_{i_2} \cdots \sigma_{i_m} L(v) 
    \label{equation: Sigma-equivariance 4}
    \\ &=
    \sigma L(v),
  \end{align}
  where for \cref{equation: Sigma-equivariance 3} to \cref{equation: Sigma-equivariance 4} we used \cref{lemma: L is transposition equivariant}.
\end{proof}

\begin{lemma}
  \label{lemma: L is order reversing}
  Let $v \in \mathbb{R}^{k}$ and $j,j' \in [k]$ be distinct such that $v_j \ge v_{j'}$. Then $[L(v)]_j \le [L(v)]_{j'}$.
  Furthermore, if $v_j > v_{j'}$, then 
$[L(v)]_j < [L(v)]_{j'}$.
\end{lemma}
\begin{proof}
  We have
  \begin{align*}
    &[L(v)]_j
    -
    [L(v)]_{j'}
    \\
    &=
    \sum_{i \in [k]: i \ne j} h(v_j - v_i)
    \\
    & \qquad 
    -
    \sum_{i \in [k]: i \ne j'} h(v_{j'} - v_i)
    \\
    &=
    h(v_j - v_{j'})+
    \sum_{i \in [k]: i \not\in  \{j, j'\}} h(v_j - v_i)
      \\
    & \qquad 
    -
    h(v_{j'} - v_j)
    - 
    \sum_{i \in [k]: i \not\in  \{j,j'\}} h(v_{j'} - v_i)
    \\
    &=
    h(v_j - v_{j'})
    -
    h(v_{j'} - v_j)
    \\
    & \qquad
    +
    \sum_{i \in [k]: i \not\in  \{j, j'\}}
    h(v_j - v_i)
    -
    h(v_{j'} - v_i).
  \end{align*}
  Since  and $h$ is monotonically non-increasing, we have
\begin{equation}
  \label{equation: monotonicity lemma}
  v_j - v_{j'} \ge 0 \ge  v_{j'} - v_j
  \implies 
  h(v_j - v_{j'}) - h(v_{j'} - v_j) \le 0
\end{equation}
  For the same reason, we have
    $h(v_j - v_i)
    -
    h(v_{j'} - v_i) \le 0.$ Putting it all together, we have $ 
    [L(v)]_j
    -
    [L(v)]_{j'}
    \le 0$, as desired.

    For the ``furthermore'' part, note that under the assumption $v_j > v_{j'}$, all inequalities in \cref{equation: monotonicity lemma} becomes strict.
\end{proof}

For reasons that will become clear later, we define for each $n \in [k-1]$
\begin{equation}
  \label{equation: WW hinge n-bayes risk}
  \underline L^n(p) := \inf_{v \in \mathbb{R}^{k} \, :\, |\argmax v| \ge n}\langle p, L(v) \rangle.
\end{equation}
Since $\argmax v$ is always nonempty, the condition that $|\argmax v| \ge 1$ is always true. Thus, we have $\underline L^1 = \underline L$.

\begin{lemma}
  \label{lemma: L bayes risk is Sigma invariant}
  For all $n \in [k-1]$, $p \in \Delta^k$ and $\sigma \in \Perm{k}$, we have $\underline{L}^n(p) = \underline{L}^n(\sigma p)$.
\end{lemma}
\begin{proof}
  Define $\mathcal{R}^{k,n} := \{v \in \mathbb{R}^k : |\argmax v| \ge n\}$.
  Since $|\argmax v | = |\argmax \sigma v|$, we have $\sigma \mathcal{R}^{k,n} = \mathcal{R}^{k,n}$.
  Introducing the change of variables $u = \sigma v$, we have
  \begin{align*}
    \underline L^n(p) &= \inf_{v \in \mathcal{R}^{k,n}}\langle p, L(v) \rangle
    \\
                    &=
                    \inf_{\sigma'u \in \mathcal{R}^{k,n}}\langle p, L(\sigma'u) \rangle
                    \quad \because \mbox{Definition of $u$}
    \\
                    &=
                    \inf_{u \in \sigma\mathcal{R}^{k,n}}\langle p, L(\sigma'u) \rangle
                    \quad \because \mbox{$\sigma^{-1} = \sigma'$}
  \\
                    &=
                    \inf_{u \in \mathcal{R}^{k,n}}\langle p, L(\sigma'u) \rangle
                    \quad \because \mbox{$\sigma \mathcal{R}^{k,n} = \mathcal{R}^{k,n}$}
  \\
                    &=
                    \inf_{u \in \mathcal{R}^{k,n}}
\langle p, \sigma'L(u) \rangle
\quad \because \mbox{\cref{corollary: L is Sigma equivariant}}
  \\
                    &=
                    \inf_{u \in \mathcal{R}^{k,n}}
\langle \sigma p, L(u) \rangle
\\
                    &= \underline{L}^n(\sigma p).
  \end{align*}
\end{proof}

  \begin{lemma}
    \label{lemma: bubble sort}
    Let $p \in \mathbb{R}^k_{\downarrow}$, $q \in \mathbb{R}^k$ be arbitrary and $\sigma \in \Perm{k}$ be such that $\sigma q \in \mathbb{R}^{k}_{\uparrow}$. Then $\langle p, q\rangle \ge \langle p, \sigma q\rangle$.
  \end{lemma}
  \begin{proof}
    Consider the ``bubble sort'' algorithm applied to $q$:
    \begin{enumerate}
      \item Initialize $q^{(0)} = q$, $t \gets 0$
      \item While there exists $i \in [k-1]$ such that $q_i^{(t)} > q_{i+1}^{(t)}$, do
        \begin{enumerate}
          \item 
        $q^{(t+1)} \gets \sigma_{(i,i+1)} q^{(t)}$
      \item $t \gets t +1$
        \end{enumerate}
      \item Output monotone non-decreasing vector $q^{(t)}$
    \end{enumerate}
    We claim that at every step, we have $\langle p, q^{(t)}\rangle \ge \langle p, q^{(t+1)}\rangle$. Let 
$a = q_i^{(t)}$ and $b= q_{i+1}^{(t)}$ as in step 2 above.
Let $c = p_i$ and $d = p_{i+1}$. Hence, we have $a > b$ and $c \ge d$. Observe that
\[
  \langle p, q^{(t)}\rangle -\langle p, q^{(t+1)}\rangle
=
ac + bd - (ad+bc)
=
(a-b)(c-d) \ge 0
\]
which proves the claim. 
Thus, we have
    \[
      \langle p, q\rangle
      =
      \langle p, q^{(0)} \rangle
      \ge 
      \langle p, q^{(1)} \rangle
      \ge 
      \cdots 
      \ge
      \langle p, q^{(t)}\rangle.
    \]
    By construction, there exists $\tau \in \Perm{k}$ such that $\tau q = q^{(t)}$.
    We must have $\tau q = \sigma q$ since both vectors are monotone non-increasing, although $\tau$ may not equal $\sigma$.
  \end{proof}

  Define the matrix $T \in \mathbb{R}^{k\times k}$
  \begin{equation}
    \label{equation: T matrix}
    T_{ij} = 
    \begin{cases}
      1 & i \ge j \\
      0 & \mbox{otherwise}.
    \end{cases}
  \end{equation}
  Also, define $D \in \mathbb{R}^{k\times k}$
  \[
    D_{ij}
    = 
    \begin{cases}
      1 &: i =j \\
      -1 &: i = j+1\\
      0 &: \mbox{otherwise.}
    \end{cases}
  \]
  In other words, $D$ is the matrix with $1$s on the main diagonal, $-1$s on the subdiagonal below the main diagonal, and $0$ everywhere else.
  We have
  \[
    [Dv]_i =
    \begin{cases}
      v_1 &: i = 1\\
      v_{i} - v_{i-1} &: i > 1.
    \end{cases}
  \]
  \begin{lemma}
    \label{lemma: T is the inverse of D}
    $D^{-1} = T$.
  \end{lemma}
  \begin{proof}
    Using Gaussian elimination for inverting a matrix, it is easy to see that $D'T'$ is the identity.
  \end{proof}

\begin{definition}
  Define the following sets:
  \[
    \mathcal{M}
    =
    \{
      v \in \mathbb{R}^{k} : v_1 = 0\mbox{ and } 0 \le v_i - v_{i+1},\, \forall [k-1]
    \},
  \]
  \[
    \mathcal{C}
    =
    \{
      v \in \mathbb{R}^{k} : v_1 = 0, v_k \le -1, \mbox{ and } 0 \le v_i - v_{i+1} \le 1, \, \forall [k-1]
    \},
  \]
  $\mathcal{M}_{\mathbb{Z}} = \mathcal{M} \cap \mathbb{Z}^k$  and $\mathcal{C}_{\mathbb{Z}} = \mathcal{C} \cap \mathbb{Z}^k$.
\end{definition}

  \begin{lemma}
    \label{lemma: cake layering representation}
    We have the following equality of sets:
    \begin{align*}
      \mathcal{M}_{\mathbb{Z}}&=
      \{
        -Tc : c \in  \mathbb{Z}_+^{k}, \, c_1 = 0
      \}
      \\
      \mathcal{C}_{\mathbb{Z}}&=
      \{
        -Ts : s \in  \{0,1\}^{k}, \, s_1 = 0,\mbox{ and }
     \exists i \in \{2,\dots, k\}: s_i = 1
      \}
    \end{align*}
  \end{lemma}
  \begin{proof}
    If $v \in \mathcal{C}_{\mathbb{Z}}$, then we have $v_i \in \mathbb{Z}_+$ and $v_i - v_{i+1} \in [0,1]$.
    These two conditions together implies that $v_{i} - v_{i+1} \in \{0,1\}$ for all $i \in [k-1]$.
    Hence, $-D v \in \{0,1\}^{k-1}$ with $[Dv]_1 = -v_1 = 0$.
    Let $-Dv = s$.
    Then \cref{lemma: T is the inverse of D} implies that $-Ts = TDv = v$.
    By construction, $s_1 = 0$.
    Furthermore, if $s_i = 0$ for all $i \in [k]$, then we would have $v = 0$ as well, which contradicts the fact that $v_k \le -1$.
    Hence, there must exists $i \in \{2,\dots, k\}$ such that $s_i = 1$.
    Clearly, all $v \in \mathcal{C}_{\mathbb{Z}}$ arise this way.
    The statement about $\mathcal{M}_{\mathbb{Z}}$ is similar.
  \end{proof}

  \begin{lemma}
    \label{lemma: margin tightening}
    Let $c \in \mathbb{Z}^k_+$ and define $s \in \{0,1\}^k$ entrywise where for each $i \in [k]$, $s_i = \mathbb{I}\{c_i \ge 1\}$.
    Then we have $[L(-Tc)]_y \ge [L(-Ts)]_y$ for all $y \in [k]$.
  \end{lemma}
  \begin{proof}
    By definition, we have
    \begin{align*}
      &[L(-Tc)]_y -[L(-Ts)]_y
      \\&= 
    \sum_{i \in [k]: i \ne y }
    h([-Tc]_y - [-Tc]_i)
    -
    h([-Ts]_y - [-Ts]_i)
    \\
                            &=
    \sum_{i \in [k]: i \ne y }
    h([Tc]_i - [Tc]_y)
    -
    h([Ts]_i - [Ts]_y)
  \end{align*}
  It suffices to show that $
    h([Tc]_i - [Tc]_y)
    -
    h([Ts]_i - [Ts]_y)
    \ge 0$ for all $i \in [k]$ such that $i \ne y$.

  First, consider when $i > y$. We have
  \[
[Tc]_i - [Tc]_y
=
\sum_{j=y+1}^{i} c_j
\]
Similarly, we have
  \[
[Ts]_i - [Ts]_y
=
\sum_{j=y+1}^{i} s_j  
=
\sum_{j=y+1}^i \mathbb{I}\{c_j \ge 1\}.
\]
From this, we see that 
\begin{align*}
[Ts]_i - [Ts]_y \ge 1 \implies 
[Tc]_i - [Tc]_y \ge 1
\\
[Ts]_i - [Ts]_y = 0 \implies 
[Tc]_i - [Tc]_y = 0.
\end{align*}
For $i > y$, we have 
$h([Ts]_i - [Ts]_y) 
=
h([Tc]_i - [Tc]_y)$.

Next, let $i < y$. We have
  \[
[Tc]_i - [Tc]_y
=
\sum_{j=i+1}^{y} -c_j.
\]
Similarly, we have
  \[
[Ts]_i - [Ts]_y
=
\sum_{j=i+1}^{y} -\mathbb{I}\{c_j \ge 1\}.
\]

Since $c_j \ge \mathbb{I}\{c_j \ge 1\}$, we have
$
[Ts]_i - [Ts]_y \ge 
[Tc]_i - [Tc]_y
$
which implies that
$h([Ts]_i - [Ts]_y) \le 
h([Tc]_i - [Tc]_y)$.  
\end{proof}

  \begin{definition}
    \label{definition: the pi mapping}
    Let $v = (v_1,\dots, v_k) \in \mathbb{R}^k$.  Define the linear map $\pi: \mathbb{R}^{k} \to \mathbb{R}^{k-1}$ 
\[
  \pi(v)  = (v_1-v_2, v_1 - v_3,\dots, v_1 - v_k).
\]
  \end{definition}

We observe that for each $i \in [k-1]$, we have
\[
  [\pi v]_{i} 
  = v_1 - v_{i+1}.
\]

\begin{definition}
  \label{definition: the rho matrices}
  Given $k \ge 2$, define the following $(k-1)$-by-$(k-1)$ square matrices $\rho_1,\rho_2,\dots, \rho_k \in \mathbb{R}^{(k-1) \times (k-1)}$:
\begin{enumerate}
  \item $\rho_1$ is the identity,
  \item Let $z = (z_1,\dots, z_{k-1}) \in \mathbb{R}^{k-1}$ be a  vector.
    For each $i >1$, define $\rho_{i}(z) \in \mathbb{R}^{k-1}$ entrywise for each $j \in [k-1]$ by
  \begin{equation}
    \label{equation: action of rho i plus 1}
    \left[\rho_{i}(z)\right]_{j} = 
  \begin{cases}
    z_j - z_{i-1} &: j \ne i-1\\
    -z_{i-1} &: j = i-1.
  \end{cases}
\end{equation}
\end{enumerate}
\end{definition}

\begin{lemma}
  [Commuting relations]
  \label{lemma: rho pi sigma relation}
  For all $i \in [k]$, we have
  $
      \pi\sigma_{i}
      =
      \rho_{i} \pi.
      $
\end{lemma}

\begin{proof}
  If $i = 1$, then $\sigma_i$ and $\rho_i$ are both identity matrices and there is nothing to show. Otherwise, suppose that $i > 1$.
  Consider $v \in \mathbb{R}^k$.
  We first calculate $\pi\sigma_{i}v$.
 For each $j \in [k-1]$, we have
  \begin{equation}
    \label{equation: commuting relations step 1}
    [\pi \sigma_i v]_{j}
    =
    [\sigma_iv]_1 - 
    [\sigma_iv]_{j+1}
    =v_i - v_{\sigma_i(j+1)}
    =
    \begin{cases}
      v_i - v_{j+1} &: i \ne j + 1
      \\
      v_i - v_1 &: i = j+1.
    \end{cases}
  \end{equation}
  Now, we compute
      $\rho_{i} \pi v$.
 Likewise, for each $j \in [k-1]$,
  \[
    [\rho_i \pi v]_j
    =
    \begin{cases}
      [\pi v]_j - [\pi v]_{i-1} &: j \ne i-1 \\
      -[\pi v]_{i-1} &: j =i-1.
    \end{cases}
  \]
  Consider the two cases above separately: for $j \ne i - 1$, we have
  \[
      [\pi v]_j - [\pi v]_{i-1} 
      =
      (v_1 - v_{j+1} )
      -
      (v_1 - v_{i})
      =
      v_i - v_{j+1}.
  \]
  On the other hand, for $i = j +1$, we have
  \[
      -[\pi v]_{i-1} 
      = -(v_1 - v_{i})
      =v_i - v_1.
    \]
    Thus, we have 
    $[\pi \sigma_i v]_{j}
    =
    [\rho_i \pi v]_j
    $
    for all $j$ which implies that 
    $\pi \sigma_i v
    =
    \rho_i \pi v
    $. Since $v$ was arbitrary, we have $\pi \sigma_i = \rho_i \pi$.
\end{proof}

\begin{definition}
  The \emph{reduced WW hinge function} $H: \mathbb{R}^{k-1} \to \mathbb{R}_{\ge 0}$ is defined as
  \[
    H(z)
    =\sum_{i=1}^{k-1} h(z_i).
  \]
\end{definition}

  \begin{definition}
    \label{definition: reduced WW hinge loss}
    For $z \in \mathbb{R}^{k-1}$, the \emph{reduced WW hinge loss} $\LL(z) \in \mathbb{R}^{k}$ is defined entrywise for each $y \in [k]$ by
\[
  [\LL(z)]_y = H(\rho_yz).
\]
  \end{definition}

  \begin{lemma}
    \label{lemma: reduced and full hinge are equivalent}
    For all $v \in \mathbb{R}^k$, we have $\LL(\pi v) = L(v)$.
  \end{lemma}
  \begin{proof}
    We first check for all $y\in [k]$ that
\begin{equation}
  \label{equation: loss reduction step 1}
  \sum_{i \in [k]\,:\, i \ne y} h(v_y - v_i)
=
H(\pi \sigma_i v).
\end{equation}
Unpacking the definition, we have
$
  H(\pi \sigma_y v) 
  =
\sum_{i \in [k-1]} h([\pi \sigma_y v]_i).
$
Now, if $y = 1$, then $[\pi v]_i = v_1 - v_{i+1}$ for all $i \in [k-1]$. Hence, \cref{equation: loss reduction step 1} holds.
If $y > 1$. Then \cref{equation: loss reduction step 1} follows from the expression for $[\pi \sigma_y v]_i$ computed in \cref{equation: commuting relations step 1}. Thus, we have proven \cref{equation: loss reduction step 1} for all $y \in [k]$.
To conclude, we have
    \begin{align}
      [L(v)]_y &= \sum_{i \in [k]\,:\, i \ne y} h(v_y - v_i)
      \\
                &= 
                H(\pi \sigma_y v)
      \\
                &=
                H(\rho_y\pi v)
      \label{equation: loss reduction application of reciprocity}
      \\
                &=
                [\LL(\pi v)]_y
    \end{align}
    where in \cref{equation: loss reduction application of reciprocity}, we applied \cref{lemma: rho pi sigma relation}.  
  \end{proof}

  \begin{lemma}
    \label{lemma: L bayes risk is optimized over integers}
    Let $n \in [k-1]$.
    If $p \in \Delta^k_\downarrow$, then
    \[
      \underline L^n(p) = \min_{v \in \mathcal{C}_{\mathbb{Z}}\,:\, v_n = 0} \langle p, L(v) \rangle.
    \]
  \end{lemma}
\begin{proof}
  Define
  \begin{equation*}
    \mathcal{N}^n = \{v \in \mathbb{R}^{k}: v_1 = \cdots = v_n  =0, \, v_i \le 0,\, \forall i \in [k]\}.
  \end{equation*}
  We first claim that
  \begin{equation}
    \label{equation: negativity}
  \underline L^n(p) = \inf_{v \in \mathcal{N}^n} \langle p, L(v) \rangle.
  \end{equation} 
  Since $\mathcal{N}^n \subseteq \{ v \in \mathbb{R}^k : |\argmax v| \ge n\}$, the ``$\le$'' part of \cref{equation: negativity} is obvious.
  For the ``$\ge$'' part, let $v \in \mathbb{R}^{k}$ be such that $|\argmax v| \ge n$.
  Then $w = v - \mathbf{1}^k\max_{i \in [k]} v_i$ is such that $w \in \mathcal{N}^n$. Furthermore, by \cref{lemma: L is translation invariant}, we have $\langle p, L(v) \rangle = \langle p ,L(w) \rangle$. Thus, we have proven the claim.
  
   Next, observe that if $v \in \mathcal{N}^n$, then
   \[
     [\pi v]_i
     =
     v_1 - v_{i+1}
     \begin{cases}
       =0 &: i \le n-1\\
       \ge 0 &: i \ge n.
     \end{cases}
  \]
  Therefore, we have
  \[
    \pi(\mathcal{N}^n)
    =
    \{ z \in \mathbb{R}^{k-1}:
    z \ge 0,\, z_i = 0, \forall i \in [n-1]\}
  \]
  where $[0] = \emptyset$.
Introducing the change of variable $z = \pi v \in \mathbb{R}^{k-1}$, we have
  \begin{align}
    \inf_{v \in \mathcal{N}^n} \langle p, L(v) \rangle
   &=
   \inf_{v \in \mathcal{N}^n} \langle p, \LL(\pi v) \rangle
    \quad
    \because
    \mbox{\cref{lemma: reduced and full hinge are equivalent}}
    \\
   &=
   \inf_{z \in \pi(\mathcal{N})} \langle p, \LL(z) \rangle\\
   &=
   \inf_{\substack{z \in \mathbb{R}^{k-1}: z \ge 0\\ z_i = 0,\, \forall i \in [n-1]}} \langle p, \LL(z) \rangle
   \label{equation: reduced LP}
  \end{align}
  
  Below, let $\mathbf{1} := \mathbf{1}^{k-1}$.
Unwinding the definition, we have 
  \[
    \langle p, \LL(z) \rangle = \sum_{i \in [k]} p_i H(\rho_i z) = \sum_{i\in [k]} p_i \mathbf{1}' \left[ \mathbf{1} - \rho_{i}z\right]_+.
  \]
  Using slack variables $\xi_i \ge \left[ \mathbf{1} - \rho_{i}z\right]_+$,  we can rewrite \cref{equation: reduced LP} as the following linear program:
  \begin{align}
 \min_{ z \in \mathbb{R}^{k-1}}\min_{(\xi_1,\dots, \xi_k) \, : \,\xi_i \in \mathbb{R}^{k-1} } \quad &\sum_{i} p_i \mathbf{1}' \xi_i
    \label{equation: WW hinge bayes risk LP}
    \\
s.t.                                        \quad                                &\xi_i \ge \mathbf{1} - \rho_{i}z\\
                                                                                 &\xi_i \ge 0, \quad \forall i \in [k] \\
                                                                                 &z \ge 0,\\
                                                                                 & z_i = 0,\, \forall i \in [n-1].
  \end{align}
  By \citet[Corollary 3.2]{bertsimas1997introduction}, for a linear programming minimization problem over a nonempty polyhedron, one of the following must be true: 1) the optimal cost is $-\infty$ or 2) a feasible minimum exists.
  Since \cref{equation: WW hinge bayes risk LP} is nonnegative and the feasible region is nonempty, a feasible minimum exists.
  Let \[
    R = \begin{bmatrix}
      \rho_1 \\ \rho_2 \\\vdots \\ \rho_k
    \end{bmatrix}
    \in \mathbb{R}^{k(k-1) \times (k-1)},
    \quad
    X = \begin{bmatrix}
      \xi_1\\ \xi_2 \\\vdots \\ \xi_k
    \end{bmatrix}
    \in \mathbb{R}^{k(k-1)},
    \quad
p\otimes \mathbf{1}
= 
\begin{bmatrix}
  p_1 \mathbf{1}\\
  p_2 \mathbf{1}\\
  \vdots\\
  p_k \mathbf{1}
\end{bmatrix}
\in\mathbb{R}^{k(k-1)}.
  \]

  We claim that
  \begin{equation}
    \label{equation: reduced LP WTS}
    \underline L^n(p) = \min_{z \in \mathbb{R}^{k-1}_+ : z_i = 0\,\forall i \in [n-1]} \langle p, \LL(z)\rangle.
  \end{equation}
  We first consider the case when $n = 1$ where we have $\underline L^1 = \underline L$.
  In this case, the linear program \cref{equation: WW hinge bayes risk LP} can be rewritten as
  \begin{align*}
    \underline{L}(p)
    = \min_{z \in \mathbb{R}^{k-1}}\min_{X \in \mathbb{R}^{k(k-1)}} \quad &(p\otimes \mathbf{1})' X
    \\
    s.t. \quad & X + R z \ge \mathbf{1}
    \\
               &X \ge 0
               \\
               & z \ge 0.
  \end{align*}
  For a positive integer $m$, let $I_m$ denote the $m\times m$ identity matrix.
  Thus,
  \begin{align}
    \min_{z \in \mathbb{R}^{k-1},\,X \in \mathbb{R}^{k(k-1)}} \quad &
    (p\otimes \mathbf{1})'X
    \label{equation: WW hinge bayes risk LP 2}
    \\
                    s.t. \quad  &
                    \underbrace{
                    \begin{bmatrix}
                      R & I_{k(k-1)} \\
                      I_{k-1}& 0 \\
                      0 & I_{k(k-1)}
                    \end{bmatrix}
                  }_{=: A}
                    \begin{bmatrix}
                      z\\
                      X
                    \end{bmatrix}
                    \ge 
                    \begin{bmatrix}
                      \mathbf{1}\\
                      0\\
                      0
                    \end{bmatrix}.
    \label{equation: WW hinge bayes risk LP 3}
  \end{align}
    We prove that $A$ is totally unimodular (TUM).
    The matrix $R$ has the property that every row has at most one $1$ and at most one $-1$, with all other entries being zeros.
    Hence, $R$ is TUM by the Hoffman's sufficient condition \citet{lawler2001combinatorial}.
    Thus, (horizontally) concatenating $R$ with an identity matrix, i.e.,
    $
    R_0 := 
      \begin{bmatrix}
        R & I_{k(k-1)}
      \end{bmatrix}
    $
    results in another TUM matrix $R_0$. Finally, $A$ is the (vertical) concatenation of $R_0$ with another identity matrix, i.e., $A = \begin{bmatrix}
      R_0 \\
      I_{k(k-1)}
    \end{bmatrix}
    $. Hence, $A$ is also TUM.

    By a well-known result in combinatorial optimization \citet{lawler2001combinatorial}, there exists an integral solution $(X^*,z^*)$ to \cref{equation: WW hinge bayes risk LP 2}.
    In particular, $z^* \in \mathbb{Z}^{k-1}_+$.
    Thus, we have proven that
    \[
      \underline{L}(p) = \langle p, \LL(z^*)\rangle = \min_{z \in \mathbb{Z}^{k-1}_+} \langle p, \LL(z) \rangle.
    \]
    This proves \cref{equation: reduced LP WTS} for the case when $n= 1$. For $n > 1$, we define the matrix $J \in \mathbb{R}^{(n-1) \times (k-1)}$ to be the first $n-1$ rows of the $(k-1)$-by-$(k-1)$ identity matrix. In other words, for $i \in [n-1]$ and $j \in [k-1]$,
    \[
      J_{ij} = \begin{cases}
        1 &: i =j \\
        0 &: i \ne j
      \end{cases}.
    \]
    Thus, we have
  \begin{align*}
      \underline L^n(p)
      =
    \min_{z \in \mathbb{R}^{k-1},\,X \in \mathbb{R}^{k(k-1)}} \quad &
    (p\otimes \mathbf{1})'X
    \\
                    s.t. \quad  &
                    \underbrace{
                    \begin{bmatrix}
                      R & I_{k(k-1)} \\
                      I_{k-1}& 0 \\
                      0 & I_{k(k-1)} \\
                      -J & 0
                    \end{bmatrix}
                  }_{=: B}
                    \begin{bmatrix}
                      z\\
                      X
                    \end{bmatrix}
                    \ge 
                    \begin{bmatrix}
                      \mathbf{1}\\
                      0\\
                      0\\
                      0
                    \end{bmatrix}.
  \end{align*}
  The matrix $B$ is formed by duplicating rows of $A$ and multiplying the duplicated row by $-1$.
  Thus, $B$ is also TUM.
  This proves \cref{equation: reduced LP WTS}.

  Below, let $z^*$ be a solution to \cref{equation: reduced LP WTS}. Define $v^* = \begin{pmatrix}
      0\\
      -z^*
    \end{pmatrix}$.
    Furthermore, $\pi(v^*) = z^*$ and so
    \begin{align*}
      \underline{L}^n(p) 
      &= \langle p, \LL(z^*)\rangle
      \\
      &= \langle p, \LL(\pi(v^*))\rangle
      \\
      &= \langle p, L(v^*)\rangle.
    \end{align*}
    Pick $\sigma \in \Perm{k}$ such that $\sigma v^* \in \mathbb{R}^k_\downarrow$.
    First we note that $L(\sigma v^*) \in \mathbb{R}^k_\uparrow$ by \cref{lemma: L is order reversing}.
    Next, by \cref{corollary: L is Sigma equivariant}, $L(\sigma v^*) = \sigma L(v^*)$.
    Hence, by \cref{lemma: bubble sort} 
    \[
      \langle p, L(v^*)\rangle \ge 
      \langle p, \sigma L(v^*)\rangle
      =
      \langle p, L(\sigma v^*)\rangle
    \]
    which implies that $\sigma v^*$ is optimal.
    Also, we observe that $\sigma v^* \in \mathcal{M}_{\mathbb{Z}}$. 
    By \cref{lemma: cake layering representation}, we can write $\sigma v^* = -Tc$ for some $c \in \mathbb{Z}^k_+$. 
    Note that since $z_1^* = \cdots = z_{n-1}^* = 0$, the vector $v^*$ has at least $n$ entries equal to $0$.
    Since $v^* \le 0$, we must have that $v_1 = \cdots = v_n^* = 0$.
    Thus, $c_1 = \cdots c_n = 0$ as well.
    Let $s \in \{0,1\}^k$ be as defined in \cref{lemma: margin tightening}. Then we have
    \[
      \underline{L}^{n}(p)
      \ge
      \langle p, L(\sigma v^*)\rangle
      =
      \langle p, L(-Tc) \rangle
      \ge
      \langle p, L(-Ts) \rangle.
    \]
    Hence, we have $\underline{L}(p) =  \langle p, L(-Ts) \rangle.$
    Since $s_i = \mathbb{I}\{c_i \ge 1\}$, we have $s_1 = \cdots = s_n = 0$ which implies that 
$[-Ts]_1 = \cdots = [-Ts]_n = 0$.
Consider the case when there exists some $i \in \{n+1,\dots,k \}$ such that $s_i = 1$, then we have $-Ts \in \mathcal{C}_{\mathbb{Z}}$ which completes the proof of \cref{lemma: L bayes risk is optimized over integers}.
Now, consider the case where there does not exists such $i$. Then we must have $s = 0$ and also $-Ts = 0$. Therefore, we have $\underline{L}^n(p) = \langle p, L(0)\rangle$.
Define $\tilde v \in \mathbb{R}^k$ entrywise by
    \[
      [\tilde v]_i = \begin{cases}
        0 &: i \ne k\\
        -1 &: i = k
      \end{cases}
    \]
    Noting that $k \in \argmin_{i \in [k]} p_i$ by the assumption that $p \in \Delta^k_\downarrow$.
    By \cref{lemma: the trivial ordered partition} below, we get that
    $\langle p, L(\tilde v) \rangle \le \langle p, L(0) \rangle$ which implies that $\langle p, L(\tilde v)\rangle= \underline{L}^n(v)$. Clearly, $\tilde v \in \mathcal{C}_{\mathbb{Z}}$ and $\tilde v_n = 0$, which implies that $\tilde v$ is feasible for the optimization in \cref{lemma: L bayes risk is optimized over integers}.
  \end{proof}

  \begin{lemma}
    \label{lemma: the trivial ordered partition}
    Let $p \in \Delta^k$ and $i^* \in \argmin_{i \in [k]} p_i$. 
    Consider the vector $\tilde v \in \mathbb{R}^k$ defined by
    \[
      [\tilde v]_i = \begin{cases}
        0 &: i \ne i^*\\
        -1 &: i = i^*
      \end{cases}
    \]
    Then 
    \begin{enumerate}
      \item $p_{i^*} \le \frac{1}{k}$
      \item $p_i = \frac{1}{k}$ for all $i$ if and only if $p_{i^*} = \frac{1}{k}$
      \item $\langle p, L(0)\rangle
        \ge \langle p, L(\tilde v)\rangle$
        with equality if and only if $p_{i^*} = \frac{1}{k}$.
    \end{enumerate}
  \end{lemma}
  \begin{proof}
    If $p_{i^*} > \frac{1}{k}$, then we would have $\sum_{i} p_i \ge k p_{i^*} > 1$, a contradiction.
    This proves that $p_{i^*} \le \frac{1}{k}$.
    For the second item, the ``only if'' direction is obvious. For the ``if'' direction, note that if $p_i > \frac{1}{k}$ for any $i$, then we again obtain $\sum_i p_i > 1$, a contradiction.
    For the third item, first observe that
\[
  [L(0)]_i
  =
  \sum_{j \in [k]: j \ne i}
  h(0)
  =
  k-1.
\]
Thus,
$L(0) = (k-1)\mathbf{1}^k$ and 
$
\langle p, L(0) \rangle
= k-1.
$
Next, we only $L(\tilde v)$. For $i \ne {i^*}$, we have
\[[L(\tilde v)]_i = \sum_{j\in [k] : j \ne i}
h(\tilde v_i - \tilde v_j)
=
  h(1)+
  \sum_{j\in [k] : j \ne i, j \ne {i^*}}
h(0)
=
k-2.
\]
When $i = {i^*}$, we have
\[
  [L(\tilde v)]_{i^*} =
  \sum_{j\in [k] : j \ne {i^*}}
h(\tilde v_{i^*} - \tilde v_j)
=
  \sum_{j\in [k] : j \ne {i^*}}
h(-1)
=2(k-1)
=
k-2 +k.
\]
From this, we deduce that
\[
  \langle p, L(\tilde v)\rangle
  =
  k-2 + kp_{{i^*}}.
\]
Therefore, we have  $p_{i^*} \le \frac{1}{k}$ and so
\[
  \langle p, L(\tilde v)\rangle
  =
  k-2 + kp_{{i^*}}
  \le k-2 + 1 = k-1 = \langle p, L(0) \rangle.
\]
Note if if $p_{i^*} < \frac{1}{k}$, then the inequality above is strict.
  \end{proof}

  \subsection{Proof of \cref{main-theorem: L bayes risk is optimized over OP}}
  \begin{proof}[Proof of \cref{main-theorem: L bayes risk is optimized over OP}]
    Recall that 
    $
  \underline L(p) = \min_{v \in \mathbb{R}^{k}} \langle p, L(v) \rangle
    $.
    Since $\mathbb{R}^{k} \supseteq \SkCZ$, we immediately have
      $\underline L(p) \le \min_{v \in \SkCZ} \langle p, L(v) \rangle$.
      Below, we focus on the other inequality.

    Pick $\sigma \in \Perm{k}$ such that $\sigma p \in \Delta^k_\downarrow$.
    By \cref{lemma: L bayes risk is optimized over integers} where $n=1$, we have
    \[
\underline{L}(\sigma p)
=
\min_{v \in \mathcal{C}_{\mathbb{Z}}}
\langle \sigma p, L(v)\rangle.
    \]
    Now, by \cref{corollary: L is Sigma equivariant}, we have
    \[
\langle \sigma p, L(v)\rangle
=
\langle p, \sigma' L(v)\rangle
=
\langle p,  L(\sigma'v)\rangle.
    \]
    Thus,
    \begin{align*}
      \underline{L}(p) &= 
\underline{L}(\sigma p)
\quad \because  \mbox{\cref{lemma: L bayes risk is Sigma invariant}}
\\
&=
\min_{v \in \mathcal{C}_{\mathbb{Z}}}
\langle p,  L(\sigma'v)\rangle
\\
&=
\min_{v \in \sigma'\mathcal{C}_{\mathbb{Z}}}
\langle p,  L(v)\rangle
\quad \because \mbox{change of variables}
\\
&\ge
\min_{v \in \SkCZ}
\langle p,  L(v)\rangle
    \end{align*}
    where for the last equality, we used the fact that $\sigma'\mathcal{C}_{\mathbb{Z}} \subseteq \SkCZ$.
  \end{proof}

\begin{lemma}
    \label{lemma: DL formula}
  Let $s \in \{0,1\}^{k}$ be such that $s_1 = 0$. Then
  \begin{equation}
    \label{equation: difference of WW loss formula}
    [DL(-Ts)]_y
    =
    \begin{cases}
      \min\{i \in [k]: s_i = 1\} - 2 &: y = 1\\
      \min\{i \in [k]: s_i = 1,\, i > y\} - 1 &: s_y = 1,\,y > 1\\
      0 &: s_y = 0,\,y > 1\\
    \end{cases}
  \end{equation}
\end{lemma}
\begin{proof}
  By the definition of $T$, we have
  \begin{equation}
    \label{equation: Ts formula}
    [Ts]_j = \sum_{i=1}^j s_i.
  \end{equation}
  First, consider the case when $y = 1$. Then by \cref{equation: Ts formula} we have $[-Ts]_1 = 0$. Furthermore,
  \begin{align*}
    [DL(-Ts)]_1
    &=
    [L(-Ts)]_1
    \\
    &=
    \sum_{i \in [k]: i \ne 1}
    h([-Ts]_1 - [-Ts]_i)
    \\
    &=
    \sum_{i \in [k]: i \ne 1}
    h( [Ts]_i)
    \end{align*}
    Note that by \cref{equation: Ts formula}, we have $[Ts]_i \ge 1$ if $i \ge \min \{j: s_j = 1\}$ and $[Ts]_i = 0$ otherwise.
    Hence, we get
    \begin{align*}
      [DL(-Ts)]_1
      &= 
      \sum_{i \in [k]: 1 < i < \min \{j: s_j = 1\}}
    h( [Ts]_i)
    \\
      &= 
      \sum_{i \in [k]: 1 < i < \min \{j: s_j = 1\}}
      1
      \\
      &=
      \min \{j \in [k]: s_j = 1\}
-2.
    \end{align*}
    This proves the first case of \cref{equation: difference of WW loss formula}.
    Below, let $y > 1$. We have
    \begin{align}
      &[DL(-Ts)]_y
      \\
      &=
      \sum_{i \in [k]: i \ne y}
      h([-Ts]_y - [-Ts]_i)
      -
      \sum_{i \in [k]: i \ne y-1}
      h([-Ts]_{y-1} - [-Ts]_{i})
      \\
      &=
      \sum_{i \in [k]: i \ne y}
      h([Ts]_i - [Ts]_y)
      -
      \sum_{i \in [k]: i \ne y-1}
      h([Ts]_i - [Ts]_{y-1})
      \\
      &=
      \sum_{i \in [k]: i < y-1}
      h([Ts]_i - [Ts]_y)
      -
      h([Ts]_i - [Ts]_{y-1})
      \label{equation: difference of WW loss formula 1}
      \\
      &\quad+
      h([Ts]_{y-1} - [Ts]_y)
      -
      h([Ts]_{y} - [Ts]_{y-1})
      \label{equation: difference of WW loss formula 2}
      \\
      &\quad+
      \sum_{i \in [k]: i > y}
      h([Ts]_i - [Ts]_y)
      -
      h([Ts]_i - [Ts]_{y-1})
      \label{equation: difference of WW loss formula 3}
    \end{align}
    If $s_y = 0$, then $[Ts]_y = [Ts]_{y-1}$ and so we have $[DL(-Ts)]_y = 0$.
    This proves the last case of \cref{equation: difference of WW loss formula}.

    Below, assume the setting of the second case, i.e., $y > 1$ and $s_y = 1$. 
    We first evaluate \cref{equation: difference of WW loss formula 1}.  
    Since $i < y-1$, we have
   \[ 
      ([Ts]_i - [Ts]_y)
      -
      ([Ts]_i - [Ts]_{y-1})
      =
      [Ts]_{y-1} - [Ts]_y
      =-1
    \]
    and
    \[
      ([Ts]_i - [Ts]_{y-1}) \le 0.
    \]

    The two preceding facts together imply that
    \[
      h([Ts]_i - [Ts]_y)
      -
      h([Ts]_i - [Ts]_{y-1})
      =1
    \]
    and so
    \[
      \sum_{i \in [k]: i < y-1}
      h([Ts]_i - [Ts]_y)
      -
      h([Ts]_i - [Ts]_{y-1})
      = y-2.
    \]
    Next, we evaluate \cref{equation: difference of WW loss formula 2}
    \[
      h([Ts]_{y-1} - [Ts]_y)
      -
      h([Ts]_{y} - [Ts]_{y-1})
      =
      h(-1)
      -
      h(1)
      =2.
    \]
    Finally, we evaluate \cref{equation: difference of WW loss formula 3}. Since $i > y$, we have
    \[
      [Ts]_i - [Ts]_y
      =
      \sum_{j= y+1}^i s_i.
    \]
    From this, we see that
    \[
      [Ts]_i - [Ts]_y
      \begin{cases}
        = 0 &: i < \min\{j \in [k]: j > y,\, s_j = 1\}\\
        \ge 1 &: \mbox{otherwise}.
      \end{cases}
    \]
    Hence,
    \[
      h([Ts]_i - [Ts]_y)
      \begin{cases}
        = 1 &: i < \min\{j \in [k]: j > y,\, s_j = 1\}\\
        = 0 &: \mbox{otherwise}.
      \end{cases}
    \]
    On the other hand,
    $
    [Ts]_i - [Ts]_{y-1}
    =
    \sum_{j=y}^i s_i \ge s_y = 1
    $
    and so $h(
    [Ts]_i - [Ts]_{y-1}) = 0$.
    Therefore,
    \begin{align*}
      &\sum_{i \in [k]: i > y}
      h([Ts]_i - [Ts]_y)
      -
      h([Ts]_i - [Ts]_{y-1})
      \\
      &=
\min\{j \in [k]: j > y,\, s_j = 1\}
-y
-1
    \end{align*}
    Putting it all together, we have
    \begin{align*}
      [DL(-Ts)]_y
      &=
      y-2 + 2
      +
\min\{j \in [k]: j > y,\, s_j = 1\}
-y
-1
\\
      &=
\min\{j \in [k]: j > y,\, s_j = 1\}
-1.
    \end{align*}
  
\end{proof}

\subsection{Proof of \cref{main-theorem: L and ell inner risk identity}}
  \begin{proof}
    [Proof of \cref{main-theorem: L and ell inner risk identity}]

Let $\mathbf{S} = (S_1,\dots, S_l) \in \mathcal{OP}_k$.
Pick $\sigma$ such that $\sigma \varphi(\mathbf{S})$ is monotonic non-increasing.  
Hence, we have 
  \[
    \sigma \varphi(\mathbf{S})
    =
    -
    [
    \underbrace{0, \dots, 0}_{|S_1|\mbox{-times}},
    \underbrace{1, \dots, 1}_{|S_2|\mbox{-times}},
    \dots,
    \underbrace{l-1, \dots, l-1}_{|S_l|\mbox{-times}}
    ].
  \]
  For each $i=1,\dots, l-1$, define $c_i(\mathbf{S}) = |S_1| + \cdots + |S_i|$.
  
  Note that 
  \begin{align}
    S_1\cup \cdots \cup S_i
    &=
    \{j \in[k]: 0 \ge [\varphi(\mathbf{S})]_j \ge -(i-1)
    \}
    \\
    &=
    \{\sigma(1),\sigma(2),\dots, \sigma(c_i(\mathbf{S}))\}.
    \label{equation: label set coherence}
  \end{align}
  Also, note that by definition, $c_i(\mathbf{S})$ is precisely the index in $[k-1]$ such that 
  \[
    \begin{cases}
      [\sigma \varphi(\mathbf{S})]_{c_i(\mathbf{S})} = -(i-1)
    \\
    [\sigma \varphi(\mathbf{S})]_{c_i(\mathbf{S})+1} = -i.
    \end{cases}
  \]
  Motivated by this, we define $\zeta(\mathbf{S}) \in \{0,1\}^{k}$ where
  \[
    [\zeta(\mathbf{S})]_j = 
    \begin{cases}
      1 &: j = c_i(\mathbf{S})+1 \mbox{ for some } i = 1,\dots, l-1\\
      0 &: \mbox{otherwise}.
    \end{cases}
  \]
  Then
  \begin{equation}
    \label{equation: towards writing pi v S as a T of s}
    \sigma \varphi(\mathbf{S}) = -T \zeta(\mathbf{S}).
  \end{equation}
  Next, note that
  \begin{align}
    \langle p, L(\varphi(\mathbf{S}))\rangle
    &=
    \langle p, L(\sigma' \sigma \varphi(\mathbf{S}))\rangle
  \label{equation: loss embedding theorem 1}
    \\
    &=
    \langle p, \sigma' L(\sigma \varphi(\mathbf{S}))\rangle
  \label{equation: loss embedding theorem 2}
    \\
    &=
    \langle \sigma p, L(\sigma \varphi(\mathbf{S}))\rangle
  \label{equation: loss embedding theorem 3}
    \\
    &=
    \langle T' (\sigma p), DL(\sigma \varphi(\mathbf{S}))\rangle
  \label{equation: loss embedding theorem 4}
    \\
    &=
    \langle T' (\sigma p), DL(-T \zeta(\mathbf{S}))\rangle
  \label{equation: loss embedding theorem 5}
  \end{align}
  where 
  \cref{equation: loss embedding theorem 1} is by $\sigma' = \sigma^{-1}$,
  \cref{equation: loss embedding theorem 2} is by \cref{corollary: L is Sigma equivariant},
  \cref{equation: loss embedding theorem 3} is a basic property of the dot product,
  \cref{equation: loss embedding theorem 4} is by \cref{lemma: T is the inverse of D},
  \cref{equation: loss embedding theorem 5} is by \cref{equation: towards writing pi v S as a T of s}.

  We first calculate $DL(-T \zeta(\mathbf{S}))$ by applying \cref{equation: difference of WW loss formula} from
    \cref{lemma: DL formula}
    to $s = \zeta(\mathbf{S})$. For the case $y=1$ of \cref{equation: difference of WW loss formula}, we have
  \begin{align*}
    [DL(-T \zeta(\mathbf{S}))]_1 
    &= 
    \min\{j \in [k-1]: [\zeta(\mathbf{S})]_j = 1\} - 2
    \\
    &= 
    c_1(\mathbf{S})+1 - 2
    \\
    &= |S_1| - 1.
  \end{align*}
  By definition, for $y > 1$, we note that $[\zeta(\mathbf{S})]_{y} = 1$ if and only if $y = c_i(\mathbf{S})+1$ for some $i\in \{ 1,\dots, l-1\}$.
  Thus,
  \begin{align*}
    [DL(-T \zeta(\mathbf{S}))]_{c_i(\mathbf{S})+1}
    &=
    \min\{j \in [k] : [\zeta(\mathbf{S})]_{j} = 1,\, j > c_i(\mathbf{S})+1\}-1
    \\
    &=
    (c_{i+1}(\mathbf{S}) + 1) -1
    =
    c_{i+1}(\mathbf{S}).
  \end{align*}
  We summarize the above as follows:
  \begin{equation*}
    [DL(-T \zeta(\mathbf{S}))]_{y}
    =
    \begin{cases}
      |S_1| - 1 &: y = 1\\
      c_{i+1}(\mathbf{S}) &: y = c_i(\mathbf{S}) + 1 \mbox{ for some $i \in [l-1]$}\\
      0 &: \mbox{otherwise}.
    \end{cases}
  \end{equation*}
  Next, we calculate $T' (\sigma p)$.
  Note that
  \begin{align*}
    [T'(\sigma p)]_y
    &= p_{\sigma(y)} + p_{\sigma(y+1)} + \cdots + p_{\sigma(k)}\\
    &= 1 - \left(p_{\sigma(1)} + \cdots + p_{\sigma(y-1)}\right).
  \end{align*}
  In particular, $[T'(\sigma p)]_1 = 1$.  Hence,
  \begin{align*}
    &\langle p, L(\varphi(\mathbf{S}))\rangle \\
    &=\langle T' (\sigma p), DL(-T \zeta(\mathbf{S}))\rangle\\
    &=
    [T'(\sigma p)]_1 (|S_1| - 1)
    \\
    &\qquad +
    \sum_{i = 1}^{l-1}
    \left([T'(\sigma p)]_{c_i(\mathbf{S})+1}\right)
    c_{i+1}(\mathbf{S})
    \\
    &= |S_1| -1
    \\
    &\qquad +
    \sum_{i = 1}^{l-1}
    \left(1 - \left(p_{\sigma(1)} + \cdots + p_{\sigma(c_i(\mathbf{S}))}\right)\right)
    c_{i+1}(\mathbf{S}).
  \end{align*}
  Recall from \cref{equation: label set coherence}
  \[
    \{\sigma(1), \sigma(2),\dots, \sigma(c_i(\mathbf{S}))\} = S_1 \cup \cdots \cup S_i.
  \]
  Hence, 
  \[
    \left(1 - \left(p_{\sigma(1)} + \cdots + p_{\sigma(c_i(\mathbf{S}))}\right)\right)
    =
    \Pr_{Y \sim p}(Y \not \in S_1 \cup \dots \cup S_i).
  \]
  Putting it all together, we have
  \begin{align*}
    \langle p, L(\varphi(\mathbf{S}))\rangle &=
|S_1|-1
+\sum_{i=1}^{l_{\mathbf{S}}-1}
|S_1 \cup \dots \cup S_{i+1}|
\Pr_{Y \sim p} (Y \not \in S_1 \cup \cdots \cup S_i)\\
                                             &=
                                             \mathbb{E}_{Y\sim p} \left[[\ell(\mathbf{S})]_Y\right]\\
                                             &=
                                             \langle p, \ell(\mathbf{S})\rangle
  \end{align*}

  This concludes the proof of 
  \cref{main-theorem: L and ell inner risk identity}.
  \end{proof}

  \section{Maximally informative losses}
  \label{section: maximally informative losses}
  We first introduce some basic properties of hyperplane arrangements that will be needed later.
  \begin{definition}
    \label{definition: hyperplane}
    A \emph{hyperplane} in $\mathbb{R}^d$ is a subset $H \subseteq \mathbb{R}^d$ of the form
    $
      H = \{ v \in \mathbb{R}^k : b- \langle a, v\rangle = 0\}
    $
    for some (column) vector $a \in \mathbb{R}^k$ and $b \in \mathbb{R}$.
  \end{definition}

  \begin{definition}
    \label{definition: hyperplane arrangement}
    Define the following:
    \begin{enumerate}
      \item 
    A \emph{hyperplane arrangement} is a set of hyperplanes $\{H_n\}_{n \in I}$ indexed by a finite set $I$.
    Let the hyperplanes be written as 
    $H_n = \{v \in \mathbb{R}^k : b^{(n)} - \langle a^{(n)}, v \rangle =0 \}$
    for each $n \in I$.
  \item 
    Define $\signvect : \mathbb{R}^k \to \{-1, 0 , 1\}^I$ entrywise by
    \[
      [
      \signvect(v)
      ]_n
      =
      \mathrm{sgn}\left(b^{(n)} - \langle a^{(n)}, v \rangle \right),
      \quad
      \mbox{ where $\forall t\in \mathbb{R},\, \mathrm{sgn}(t) = \begin{cases}
          1 &: t > 0\\
          0 &: t = 0\\
          -1 &: t < 0
      \end{cases}$ .}
    \]
  \item 
    Define the set $\Theta := \signvect(\mathbb{R}^k) \subseteq  \{-1,0,1\}^I$.
  \item 
    For each $\theta \in \Theta$, define
    \[
      \tilde P_\theta :=
      \signvect^{-1}(\theta) =
      \{ v \in \mathbb{R}^k: \signvect(v) = \theta\}
      \quad \mbox{ and }
      P_\theta := \mathrm{cl}(\tilde P_\theta)
    \]
    where $\mathrm{cl}$ denotes the closure of a set in $\mathbb{R}^k$ with the Euclidean topology.
    \end{enumerate}
  \end{definition}

  \begin{definition}
    An \emph{affine subspace} of $\mathbb{R}^k$ is a set of the form $W + v$ where $W \subseteq \mathbb{R}^k$ is a linear subspace and $v \in \mathbb{R}^k$ is a vector.
    Let $C$ be a convex set.
    The \emph{affine hull} $\mathrm{Aff}(C)$ of $C$ is defined as the smallest affine subspace containing $C$.
    The \emph{relative interior} of $C$, denoted $\mathrm{relint}(C)$, is defined as the subset of $v \in C$ such that for all $\epsilon >0$ sufficiently small, we have that
    \[
      \mathrm{Aff}(C) \cap \{ w  \in \mathbb{R}^k: \|w - v \| < \epsilon\} \subseteq 
      C.
    \]
    In other words, $\mathrm{relint}(C)$ is an open subset of $\mathrm{Aff}(C)$.
    Here $\|\bullet\|$ is the Euclidean 2-norm on $\mathbb{R}^k$.
  \end{definition}

  The following result is ``folklore''. Since we cannot find its proof, we prove it here.
  \begin{lemma}
    \label{lemma: folklore result about hyperplane arrangement}
    Let $\{H_n\}_{n \in I}$ be an arrangement of hyperplanes. Adopt all notations from 
    \cref{definition: hyperplane arrangement}. The following are true:
    \begin{enumerate}
      \item \label{lemma-part: folklore open cell}
    For all $\theta \in \Theta$,
    $
      \tilde P_\theta
      =
      \left\{ v \in \mathbb{R}^k:
        \begin{cases}
          \theta_n (b^{(n)}- \langle a^{(n)}, v \rangle)  > 0 &: \theta_{n} \ne 0 \\
          b^{(n)}- \langle a^{(n)}, v \rangle  = 0 &: \theta_{n} = 0\\
        \end{cases}
        ,\, \forall n \in I
      \right\}
    $,
      \item \label{lemma-part: folklore closed cell}
    For all $\theta \in \Theta$,
    $
      P_\theta
      =
      \left\{ v \in \mathbb{R}^k:
        \begin{cases}
          \theta_n (b^{(n)}- \langle a^{(n)}, v \rangle)  \ge 0 &: \theta_{n} \ne 0 \\
          b^{(n)}- \langle a^{(n)}, v \rangle  = 0 &: \theta_{n} = 0\\
        \end{cases}
        ,\, \forall n \in I
      \right\}
    $,
      \item \label{lemma-part: folklore open cell is relint of closed cell}
    For all $\theta \in \Theta$,
      $\mathrm{relint}(P_\theta) = \tilde P_\theta$,
      \item \label{lemma-part: folklore disjoint union of open cell is Rk}
      $\bigsqcup_{\theta \in \Theta} \mathrm{relint}(P_\theta) = \mathbb{R}^{k}$ as a disjoint union.
    \end{enumerate}
  \end{lemma}
  \begin{proof}
    First, we note that \cref{lemma-part: folklore open cell} follows directly from definition.

    For \cref{lemma-part: folklore closed cell}, let $Q_\theta$ denote the set on the right hand side of the identity.
    We want to show that $P_\theta = Q_\theta$.
    Recall that $P_\theta = \mathrm{cl}(\tilde P_\theta)$ is by definition the smallest closed set containing $\tilde P_\theta$.
    Clearly, $Q_\theta$ is a closed set.
    Furthermore, by \cref{lemma-part: folklore open cell}, we have $\tilde P_\theta \subseteq Q_\theta$.
    Thus, we have the $P_\theta \subseteq Q_\theta$.

    Conversely, let $v \in Q_\theta$ and $w \in \tilde P_\theta$.
    Then by \cref{lemma-part: folklore open cell}, we have that $(1-\lambda) w + \lambda v \in \tilde P_\theta$ for all $\lambda \in [0,1)$.
    Now, $\lim_{\lambda \to 1} (1-\lambda) w + \lambda v = v$.
    Since $\mathrm{cl}(\tilde P_\theta)$ is closed, it contains all limits. Hence $v \in \mathrm{cl}(\tilde P_\theta) = P_\theta$, as desired.
    This proves that $Q_\theta \subseteq P_\theta$, as desired.

    Next, we prove \cref{lemma-part: folklore open cell is relint of closed cell}. From the first paragraph of \citet[Section 1.1.6.D]{ben2020optimization}, we have
    $
    \mathrm{relint}(\tilde P_\theta) \subseteq \tilde P_\theta \subseteq \mathrm{cl}(\tilde P_\theta)$.
    By \citet[Theorem 1.1.1 (iv)]{ben2020optimization}, we have
    $\mathrm{relint}(\tilde P_\theta) = 
    \mathrm{relint}(\mathrm{cl}(\tilde P_\theta))
    $. By definition $P_\theta = \mathrm{cl}(\tilde P_\theta)$. Putting it all together, we get
    $\mathrm{relint}(P_\theta) \subseteq \tilde P_\theta$.

    For the other inclusion, let $v \in \tilde P_\theta$.
    Let 
    \[W = \{ v \in \mathbb{R}^k : b^{(n)} - \langle a^{(n)}, v\rangle = 0,\, \forall n \in I \mbox{ such that } \theta_n = 0\}.\]
    Then by \cref{lemma-part: folklore closed cell}, $W$ is an affine subspace containing $P_\theta$.
    Thus, by definition of the affine hull, we have $W \supseteq \mathrm{Aff}(P_\theta)$.
    Furthermore, by \cref{lemma-part: folklore open cell}, we have, for all $\epsilon > 0$ sufficiently small, that
    $
      W \cap \{ w  \in \mathbb{R}^k: \|w - v \| < \epsilon\} \subseteq 
      P_\theta$. 
      This proves that $v \in \mathrm{relint}(P_\theta)$ and so $\tilde P_\theta \subseteq \mathrm{relint}(P_\theta)$.

      Finally, we prove \cref{lemma-part: folklore disjoint union of open cell is Rk}
    \begin{equation*}
    \bigsqcup_{\theta \in \Theta} \mathrm{relint}(P_\theta)
    =
    \bigsqcup_{\theta \in \Theta} \tilde P_\theta
    =
    \bigsqcup_{\theta \in \signvect(\mathbb{R}^k)} \signvect^{-1}(\theta)
    =
    \mathbb{R}^k,
  \end{equation*}
  where for the middle equality, we recall that $\Theta = \signvect(\mathbb{R}^k)$ by definition.
  \end{proof}

  \subsection{Semiordered hyperplane arrangement}
  
  Below, we apply the results of \cref{lemma: folklore result about hyperplane arrangement} to the ``semiorder hyperplane arrangement'', which is closely connected to the WW-hinge loss.

\begin{definition}
  \label{definition: semiorder HA}
  The \emph{semiorder hyperplane arrangement} is the hyperplane arrangement in $\mathbb{R}^k$ indexed by the finite set
  $I = \{(i,j) \in [k] \times [k]: i \ne j\}$ with
  the $(i,j)$-th hyperplane given by $H_{(i,j)} = \{v \in \mathbb{R}^k : 1- (v_i - v_j) = 0\}$.
\end{definition}

  \begin{lemma}
    \label{lemma: semiorder HA key property}
    Let $L: \mathbb{R}^k \to \mathbb{R}^k_+$ be the WW-hinge loss and $\SkCZ$ be as in \cref{main-definition: SkCZ}.
    Let $\{H_{(i,j)}\}_{(i,j) \in I}$ be the semiorder hyperplane arrangement as in \cref{definition: semiorder HA}.
    Adopt all notations from \cref{definition: hyperplane arrangement}.
    Then we have for all $\theta \in \Theta$ that
  \begin{enumerate}
  \item the restriction of $L$ to $P_\theta$, denoted $L|_{P_\theta}$, is an affine function,
\item 
$P_\theta \cap \SkCZ$ is nonempty.
  \end{enumerate}
  \end{lemma}
  \begin{proof}
    For the first item, fix some $i \in [k]$ and note that
    \[
      [L(v)]_i
      =
      \sum_{j \in [k] : j \ne i}
      \max\{0, 1- (v_i - v_j)\}.
    \]
    Fix $(i,j) \in I$ where $I$ is as in \cref{definition: semiorder HA}.
    Then by
    \cref{lemma: folklore result about hyperplane arrangement} item 2,
    for all $v \in P_\theta$, we have
    \[
      \max\{0, 1- (v_i - v_j)\} 
      =
      \begin{cases}
        1 - (v_i - v_j)  &: \theta_{(i,j)} = 1\\
        0 &: \mbox{otherwise.}
      \end{cases}
    \]
    In either case,  $
      \max\{0, 1- (v_i - v_j)\}$ is affine over $P_\theta$.

      Next, we prove the second item.
      Define
      $H_0 = \{v \in \mathbb{R}^k : \sum_{i \in [k]} v_i = 0\}$.
      Then $H_0 \cap P_\theta$ is nonempty for all $\theta \in \Theta$.
      To see this, first note that $P_\theta$ is nonempty by construction.
      Furthermore, if $v \in P_\theta$ then $v + c\mathbf{1}^k \in P_\theta$ as well for any $c \in \mathbb{R}$.
      Thus, $v+ (-(1/k)\sum_{i \in [k]} v_i ) \mathbf{1}^k \in H_0 \cap P_\theta$.

      \begin{lemma}
        $H_0 \cap P_\theta$ does not contain any line.
      \end{lemma}
      \begin{proof}
      Suppose that this is false, i.e., $\mathfrak{l} \subseteq H_0 \cap P_\theta$ where $\mathfrak{l} \subseteq \mathbb{R}^k$ is a line.
      In particular, $\mathfrak{l} \subseteq H_0$.
      This means that $\mathfrak{l} = \{cw : c \in \mathbb{R}\}$ where $w \in H_0$ is a nonzero vector.
      Thus, there exists $i \ne j$ such that $w_i > 0$ and $w_j < 0$.
      Recall from \cref{definition: hyperplane arrangement} that
      $[\signvect(cw)]_{(i,j)}
      =
      \mathrm{sgn}\left(1- c(w_i - w_j)\right)$. 
      Thus, as $c$ ranges over $\mathbb{R}$, we have that  $[\signvect(cw)]_{(i,j)}$ takes on all three values in $\{-1,0, 1\}$.
      However, by \cref{lemma: folklore result about hyperplane arrangement} item 2, 
      $[\signvect(cw)]_{(i,j)}$ can only take on at most two distinct values in $\{-1,0,1\}$.
      \end{proof}
      Before proceeding, we recall a definition:
      \begin{definition}
        \label{definition: polyhedron and basic feasible solution}
        A \emph{polyhedron} $P$ in $\mathbb{R}^k$ is a set of the form
        $P = \{ x \in \mathbb{R}^k: \langle a^{(n)}, x \rangle \le b^{(n)},\,\forall n \in [m]\}$ where 
        $m$ is a positive integer, $a^{(n)} \in \mathbb{R}^k$ and $b^{(n)} \in \mathbb{R}$ for all $n \in [m]$.
        For each $n \in [m]$, the tuple $(a^{(n)},b^{(n)})$ is called a \emph{constraint} of $P$.
        A point $x \in P$ is a \emph{basic feasible solution} (BFS) if there exists $n_1,\dots, n_k \in [m]$ such that 
        \begin{enumerate}
          \item 
            $\langle a^{(n_i)}, x\rangle = b^{(n_i)}$ for all $i \in [k]$, and 
          \item $\mathcal{A} := \{a^{(n_1)},\dots, a^{(n_k)}\}$ is a basis for $\mathbb{R}^k$.
        \end{enumerate}
      \end{definition}

      By \citet[Theorem 2.6]{bertsimas1997introduction}
      and \cite[Theorem 2.3]{bertsimas1997introduction},
      a polyhedron which does not contain any line always have a BFS.
      Earlier, we proved that $H_0 \cap P_\theta$ does not contain any line.
      Hence, $H_0 \cap P_\theta$ contains a BFS.
      For the remainder of this proof, let $x \in \mathbb{R}^k$ be such a BFS 
      with associated basis $\mathcal{A} = \{a^{(n_1)},\dots, a^{(n_k)}\}$ as in \cref{definition: polyhedron and basic feasible solution}.


      Let $e^{i} \in \mathbb{R}^k$ be the $i$-th elementary basis vector in $\mathbb{R}^k$.
      By definition of $P_\theta \cap H_0$, we have 
      \[
        \mathcal{A} 
        \subseteq 
        \{e^i - e^j : (i,j) \in I \} 
        \cup \{\mathbf{1}^k\}
      \]
      where we recall that $I$ is as in \cref{definition: semiorder HA}.
      Observe that $\langle \mathbf{1}^k, e^{i} - e^{j}\rangle = 0$ for all $(i,j) \in I$.
      Hence, we must have that $\mathbf{1}^k \in \mathcal{A}$, since  otherwise $\mathcal{A}$ cannot span $\mathbb{R}^k$.
      This implies that we necessarily have $\mathbf{1}^k \in \mathcal{A}$. Without the loss of generality, let $a^{(n_k)} = \mathbf{1}^k$.
      Since $\mathcal{A}$ is linearly independent, we have 
      \[
        \mathcal{B}:=
        \mathcal{A} \setminus \{a^{(n_k)}\} = 
        \{a^{(n_1)},\dots, a^{(n_{k-1})}\}
        \subseteq 
        \{e^i - e^j : (i,j) \in I \}.
      \]

      Now, for each $i \in [k-1]$, let
      $(t_i,h_i) \in I$ be such that
      $a^{(n_i)} = e^{t_i} - e^{h_i}$.
      By the definition of $P_\theta$, we have
      $\langle a^{(n_i)}, x \rangle = x_{t_i} - x_{h_i} =\pm 1$.
      Note that this implies that $x$ is not a scalar multiple of $\mathbf{1}^k$.

      Next, consider the directed graph $G$ with vertices $V(G) = [k]$ and edges are $E(G) = 
      \{(t_i,h_i): i \in [k-1]\}$.
      Since $\mathcal{B}$ is linearly independent, we observe that if $(t_i, h_i) \in E(G)$, then $(h_i,t_i) \not \in E(G)$.
      Let $G^u$ be the undirected graph obtained from $G$ by forgetting the edge orientations. 
      By the preceding observation, we have $|E(G^u)| = k-1$.
      An undirected edge is denoted as $\{\alpha,\beta\} \in E(G^u)$.

      Observe that if $\{\alpha,\beta\} \in E(G^u)$, then $x_\alpha - x_\beta = \pm 1$.

      \begin{lemma}
        $G^u$ is a tree, i.e., a connected graph without cycles.
      \end{lemma}
      \begin{proof}
      Note that
      $G^u$ does not contain any cycles. To see this, note that if $G^u$ had a cycle, then $\mathcal{A}$ cannot be linearly independent.
      Thus, $G^u$ is a disjoint union of trees $\{T_1,\dots, T_f\}$ where $f$ is a positive integer.
      Since each $T_i$ is a tree, we have $|E(T_i)| = |V(T_i)| - 1$.
      On the other hand, we have
      \begin{align*}
        k-1 &= |E(G^u)| \\
            &= |E(T_1)| + \cdots + |E(T_f)|\\
            & = |V(T_i)| + \cdots + |V(T_f)| - f\\
            & = |V(G^u)| - f\\
            & = k - f
      \end{align*}
      which implies that $f = 1$. In other words, $G^u$ is a tree to begin with.
      \end{proof}

      Although we know that $G^u$ is a tree, we only need the fact that $G^u$ is connected.

      Let $\alpha,\beta \in V(G^u)$. A \emph{path} of length $l$ from $\alpha$ to $\beta$ is a sequence $\phi_1,\dots, \phi_l \in V(G^u)$ such that
      \begin{enumerate}
        \item $\phi_1 = \alpha$ and $\phi_l = \beta$
        \item $\{\phi_i, \phi_{i+1}\} \in E(G^u)$ for all $i \in [m-1]$.
      \end{enumerate}
      The fact that $G^u$ is connected implies that there exists a path between any two vertices $\alpha,\beta \in V(G^u)$.
      Define $\overline{x} := \max x$ and $\underline{x} := \min x$.
      \begin{lemma}
        \label{lemma: SHPA property 1}
        For all $\beta \in [k]$, we have $\overline{x} - x_\beta \in \mathbb{Z}$.
      \end{lemma}
      \begin{proof}
        Let $\alpha \in \argmax x$ and consider a path $\phi_1,\dots, \phi_l \in V(G^u)$ from $\alpha$ to $\beta$.
       Observe that 
      $
      x_\alpha - x_\beta 
      =
      \sum_{i \in [l-1]}
      x_{\phi_i} - x_{\phi_{i+1}}
      $. Since $\{\phi_i, \phi_{i+1}\} \in E(G^u)$, we have 
      $
      x_{\phi_i} - x_{\phi_{i+1}} = \pm 1$.
      This proves that $x_\alpha - x_\beta \in \mathbb{Z}$.
      \end{proof}
      Let $D := \overline{x} - \underline {x}$.
      Since $x_\beta \ge \underline x$, we have $0 \le \overline x - x_\beta \le D$.
      Apply \cref{lemma: SHPA property 1} with $\beta \in \argmin x$, we get 
      $\overline x-\underline x = D\in \mathbb{Z}$.
      In summarize, we have proven that
      \begin{equation}
        \label{equation: BFS is almost SkCZ}
        \{x_\beta - \overline x : \beta \in [k]\}
        \subseteq 
        \{-D, -D+1,\dots, -1,0\}.
      \end{equation}
      Below, we will show that the inclusion in \cref{equation: BFS is almost SkCZ} is in fact an equality.

      Next, let  $\overline{\varrho} \in \argmax x$ and $\underline{\varrho} \in \argmin x$.
      Let $\phi_1,\dots, \phi_l \in V(G^u)$ be a path between $\overline{\varrho}$ and $\underline{\varrho}$.       Note that by definition we have 
      \begin{enumerate}
        \item $x_{\phi_1} = \overline x$ and $x_{\phi_l} = \underline x$,
        \item $x_{\phi_i} - x_{\phi_{i+1}} =\pm 1$ for all $i \in [l-1]$.
      \end{enumerate}
Consider the sequence of numbers
      \[
        S:=
        ( \underbrace{x_{\phi_1} - \overline{x}}_{=-D},
        x_{\phi_2} -\overline{x},
        \dots,
        x_{\phi_{l-1}} - \overline{x},
        \underbrace{x_{\phi_l} - \overline{x}}_{=0}).
      \]
      Notice that the difference between consecutive entries of $S$ is $\pm 1$.
      Thus, the sequence $S$ takes on every value in $\{-D,-D+1,\dots, -1,0\}$ at least once.
      This proves that \cref{equation: BFS is almost SkCZ} holds with equality, i.e.,
      \begin{equation}
        \label{equation: BFS is SkCZ}
        \{x_\beta - \overline x : \beta \in [k]\}
        =
        \{-D, -D+1,\dots, -1,0\}.
      \end{equation}
      Now, let $\sigma \in \Perm{k}$ be the element such that $\sigma x$ is monotonic non-increasing.
      Earlier, we argued that $x$ is not a scalar multiple of $\mathbf{1}^k$.
      Thus, \cref{equation: BFS is SkCZ} implies that $\sigma x - \overline x\mathbf{1}^k \in \mathcal{C}_{\mathbb{Z}}$.
      Consequently, we have $x - \overline x \mathbf{1}^k \in \SkCZ$.
      Since $x \in P_\theta$, we have $x - \overline x \mathbf{1}^k \in P_\theta$ as well.
      This proves that $P_\theta \cap \SkCZ$ is nonempty, which concludes the proof of \cref{lemma: folklore result about hyperplane arrangement}.
  \end{proof}

  \subsection{Proof of \cref{main-proposition: maximal informative}}
  \begin{proof}[Proof of \cref{main-proposition: maximal informative}]
  Let $m = |\mathcal{OP}_k|$. Index the elements of $\mathcal{OP}_k$ by $[m]$, i.e.,
  \[
    \mathcal{OP}_k = \{\mathbf{S}^1,\dots, \mathbf{S}^m\}.
  \]
  For each $i \in [m]$, let $p^{(i)} \in \Delta^k$ be such that 
  $\{\mathbf{S}^i\} = \argmin_{\mathbf{S} \in \mathcal{OP}_k}
  \langle p, \ell(\mathbf{S})\rangle$. 
  The existence of such $p^{(i)}$s was confirmed by computer search for $k \in \{3,\dots, 7\}$.
  Equivalently, $\mathbf{S}^{i}$ is the unique element of $\mathcal{OP}_k$ such that
  \begin{equation}
    \label{equation: proof of max info embedding uniqueness}
    \langle p^{(i)}, \ell(\mathbf{S}^{i})\rangle = \underline{\ell}(p^{(i)})
    =
    \underline{L}(p^{(i)})
  \end{equation}
  where the second equality is by 
  \cref{main-lemma: L and ell bayes risk are equal}.

  Next, suppose $L$ embeds another discrete loss $\lambda: \mathcal{R} \to \mathbb{R}^k_+$ with embedding map $\chi: \mathcal{R} \to \mathbb{R}^k$.
  Our goal is to show that $|\mathcal{R}| \ge |\mathcal{OP}_k|$.
  To this end, let $\mathcal{R} = \{r^1,\dots, r^n\}$.
  Since $L$ embeds $\lambda$ via $\chi$, we have by definition that
  $\underline{L}(p) = \underline{\lambda}(p)
  = \min_{r \in \mathcal{R}} \langle p, L(\chi(r))\rangle$.
  In particular, for a fixed $i \in [m]$, there exists $\iota(i) \in [n]$ such that
  $\underline{L}(p^{(i)}) = 
  \langle p^{(i)}, L(\chi(r^{\iota(i)}))\rangle$.
  Note that this defines a mapping
  \begin{equation}
    \label{equation: proof of max info embedding iota mapping}
    \iota: [m] \to [n].
  \end{equation}
  Let $v^{(i)}:= \chi(r^{\iota(i)})$. Combined with \cref{equation: proof of max info embedding uniqueness}, we have
  \begin{equation}
    \label{equation: proof of max info embedding 1}
    \langle p^{(i)}
    ,
    L(v^{(i)})\rangle
    =
    \underline{L}(p^{(i)})
    =
    \underline{\ell}(p^{(i)}).
  \end{equation}
  Consider $\{P_\theta\}_{\theta \in \Theta}$ as in \cref{lemma: semiorder HA key property}.
  For each $v \in \mathbb{R}^k$, let $\theta(v) \in \Theta$ be the unique element such that $v \in \mathrm{relint}\left(P_{\theta(v)}\right)$.
  The existence and uniqueness of $\theta(v)$ is guaranteed by \cref{lemma: folklore result about hyperplane arrangement} \cref{lemma-part: folklore disjoint union of open cell is Rk}.
  
  By \cref{equation: proof of max info embedding 1}, we have $v^{(i)} \in \argmin_{v \in \mathbb{R}^k} \langle p^{(i)}, L(v)\rangle$.
  By \cref{lemma: semiorder HA key property}, the function $v \mapsto \langle p^{(i)}, L(v)\rangle$ is affine over the domain $P_{\theta(v^{(i)})}$.
  Furthermore, it is minimized at $v^{(i)} \in \mathrm{relint}(P_{\theta(v)})$.
  Thus, by \citet[Lemma 1.2.2]{ben2020optimization}, the function $v \mapsto \langle p^{(i)}, L(v)\rangle$ is constant over the domain $v \in P_{\theta(v^{(i)})}$.
  Since $v^{(i)} \in P_{\theta(v^{(i)})}$
 and $ 
    \langle p ^{(i)}, L(v^{(i)})\rangle
    =
    \underline{L}(p^{(i)})$ by \cref{equation: proof of max info embedding 1}, we have 
    \begin{equation} 
      \label{equation: proof of max info embedding L is locally constant}
    \langle p ^{(i)}, L(v)\rangle
    =
    \underline{L}(p^{(i)})
    , \, \forall v \in P_{\theta(v^{(i)})}
  \end{equation}
  Next, recall that 
  $P_{\theta} \cap \SkCZ$ is nonempty for all $\theta \in \Theta$.
  In particular, 
  $
    P_{\theta(v^{(i)})} \cap \SkCZ 
    $ is nonempty. 
    By \cref{main-theorem: bijection between OP representations}, we have $\SkCZ = \varphi(\mathcal{OP}_k)$.
    All elements of 
    $P_{\theta(v^{(i)})} \cap \SkCZ$ are of the form
    $\varphi(\mathbf{S})$ for some $\mathbf{S} \in \mathcal{OP}_k$.
    Fix such an $\mathbf{S}$ so that $\varphi(\mathbf{S}) \in P_{\theta(v^{(i)})} \cap \SkCZ$.
    Now,
    \[
      \langle p ^{(i)}, L(\varphi(\mathbf{S}))\rangle
      \overset{\mbox{\cref{equation: proof of max info embedding L is locally constant}}}
{ = }
    \underline{L}(p^{(i)})
  \overset{\mbox{\cref{equation: proof of max info embedding 1}}}
  { =}
    \underline{\ell}(p^{(i)}).
    \]
    Recall from right before \cref{equation: proof of max info embedding uniqueness}, we have that $\mathbf{S}^i$ is the unique element of $\mathcal{OP}_k$ such that
    $ \langle p ^{(i)}, L(\varphi(\mathbf{S}^{(i)}))\rangle = 
    \underline{\ell}(p^{(i)})$. This proves that $\mathbf{S} = \mathbf{S}^i$. Thus, we have shown that
  \begin{equation}
    \label{equation: maximally informative proof intersection}
    P_{\theta(v^{(i)})} \cap \SkCZ 
    =
    \{\varphi(\mathbf{S}^i)\}.
  \end{equation}
  Finally, we are now ready to prove that $n=|\mathcal{R}| \ge |\mathcal{OP}_k| = m$.
  It suffices to show that the mapping $\iota: [m] \to [n]$ defined at
    \cref{equation: proof of max info embedding iota mapping} is injective.
  Suppose that there exists distinct $i,j \in [m]$ such that $\iota(i) = \iota(j)$.
  Then
\begin{align*}
  & r^{\iota(i)} = r^{\iota(j)}
  \\ \implies &
  v^{(i)} = v^{(j)} \quad \because \mbox{definition of $v^{(i)} := \chi(r^{\iota(i)})$}
  \\ \implies &
  \theta(v^{(i)}) = \theta(v^{(j)})
  \\ \implies &
    P_{\theta(v^{(i)})} \cap \SkCZ = P_{\theta(v^{(j)})} \cap \SkCZ  
  \\ \implies &
  \{\varphi(\mathbf{S}^i)\} = \{\varphi(\mathbf{S}^j)\} \quad \because \mbox{\cref{equation: maximally informative proof intersection}}
  \\ \implies &
    \varphi(\mathbf{S}^i) = \varphi(\mathbf{S}^j)
  \\ \implies &
  \mathbf{S}^i = \mathbf{S}^j \quad \because \mbox{$\varphi$ is a bijection}
  \end{align*}
  which contradicts $i\ne j$. Thus, we have that $\iota : [m] \to [n]$ is injective which implies that $n \ge m$.
\end{proof}

\section{The argmax link}
\label{section: the argmax link}

\begin{definition}
  For $\sigma \in \Perm{k}$ and $\mathbf{S} = (S_1,\dots, S_l) \in \mathcal{OP}_k$, define $\sigma(\mathbf{S}) \in \mathcal{OP}_k$ by
  \[
    \sigma(\mathbf{S})
    =
    (\sigma(S_1),\dots, \sigma(S_l))
  \]
  where $\sigma(S_i) = \{ \sigma(j) : j \in S_i\}$ for each $i \in [l]$.
\end{definition}

\begin{lemma}
  \label{lemma: varphi is Sigma t-equivariant}
  For $\sigma \in \Perm{k}$ and $\mathbf{S} = (S_1,\dots, S_l) \in \mathcal{OP}_k$,
  we have
  \[
    \sigma' \varphi(\mathbf{S})
    = \varphi(\sigma(\mathbf{S})).
  \]
\end{lemma}
\begin{proof}
  By definition, we have
  \[
    [\varphi(\sigma(\mathbf{S}))]_j
    =
    -(i-1),\, \forall j \in \sigma(S_i).
  \]
  Since $j \in \sigma(S_i) \iff \sigma^{-1}(j) \in S_i$, we have
  \[
    [\varphi(\sigma(\mathbf{S}))]_j
    =
    -(i-1),\, \forall j \in [k] : \sigma^{-1}(j) \in S_i.
  \]
  Introduce the change of variable $m = \sigma^{-1}(j)$, we have
  \[
    [\varphi(\sigma(\mathbf{S}))]_{\sigma(m)}
    =
    -(i-1),\, \forall m \in S_i.
  \]
  On the other hand, we have
  \[
    [\sigma' \varphi (\mathbf{S})]_{\sigma(m)}
    =
    [\varphi(S)]_{\sigma' \sigma(m)}
    =
    [\varphi(S)]_{m}
    =
    -(i-1), \, \forall m \in S_i.
  \]
  This proves that $\sigma' \varphi(\mathbf{S}) = \varphi(\sigma \mathbf{S})$.
\end{proof}

\begin{corollary}
  \label{lemma: ell is Sigma equivariant}
  For all $\mathbf{S} \in \mathcal{OP}_k$ and $\sigma \in \Perm{k}$, we have
  $\sigma \ell(\mathbf{S}) = \ell(\sigma' \mathbf{S})$.
\end{corollary}
\begin{proof}
  Since $\Delta^k$ spans $\mathbb{R}^k$, it suffices to check that $\langle p, \sigma \ell(\mathbf{S})
  \rangle
  =
  \langle p, \ell(\sigma ' \mathbf{S})\rangle$ for all $p \in \Delta^k$. 
  To this end, we have
  \begin{align*}
\langle p, \ell(\sigma ' \mathbf{S})\rangle
  &=
  \langle p, L(\varphi(\sigma ' \mathbf{S}))\rangle
  \quad \because \mbox{\cref{main-theorem: L and ell inner risk identity}}
  \\
  &=
  \langle p, L(\sigma \varphi( \mathbf{S}))\rangle
  \quad \because \mbox{\cref{lemma: varphi is Sigma t-equivariant}}
  \\
  &=
  \langle p, \sigma L(\varphi( \mathbf{S}))\rangle
  \quad \because \mbox{\cref{corollary: L is Sigma equivariant}}
  \\
  &=
  \langle \sigma' p, L(\varphi( \mathbf{S}))\rangle
  \\
  &=
  \langle \sigma' p, \ell( \mathbf{S})\rangle
  \quad \because \mbox{\cref{main-theorem: L and ell inner risk identity}}
  \\
  &=
  \langle p, \sigma \ell( \mathbf{S})\rangle
  \end{align*}
  as desired.
\end{proof}

For $p \in \Delta^k$, define
\begin{align}
  \gamma(p)
  &:=
  \argmin_{\mathbf{S} \in \mathcal{OP}_k} \langle p, \ell(\mathbf{S}) \rangle,\\
  \Gamma(p)
  &:=
  \argmin_{
  v \in \mathbb{R}^{k}}
  \langle p,
  L(v)
  \rangle.
\end{align}

\begin{lemma}
  \label{lemma: monotonic non-increasing p}
  Let $p \in \Delta^k_\downarrow$, $v \in \Gamma(p)$, and $\sigma$ be such that $\sigma v \in \mathbb{R}^k_\downarrow$.
  Then $\sigma p  = p$ and $\sigma v \in \Gamma(p)$.
\end{lemma}
\begin{proof}
  Let $i \in [k-1]$ be such that 
  $v_i < v_{i+1}$.
  We first prove that $p_i = p_{i+1}$.
  Let $\tau = \sigma_{(i,i+1)}$. Since $\tau$ is a transposition, we have $\tau' = \tau$. Now,
  \begin{align*}
    0 &\le
    \langle p , L( \tau v) \rangle - \langle p , L( v) \rangle
    \quad \because \mbox{Optimality of $v$}
    \\
      &
      =
    \langle p , \tau L( v) \rangle - \langle p , L( v) \rangle
    \quad \because \mbox{\cref{corollary: L is Sigma equivariant}}
    \\
      &
      =
    \langle \tau p , L( v) \rangle - \langle p , L( v) \rangle
    \quad \because \mbox{$\tau' = \tau$.}
    \\
      &
      =
      (p_{i+1} - p_i) [L(v)]_i
      +
      (p_i - p_{i+1}) [L(v)]_{i+1}
      \\
      &
      =
      (p_{i+1} - p_i)
      ( [L(v)]_i - [L(v)]_{i+1})
  \end{align*}
  By 
  \cref{lemma: L is order reversing}, we have
  $
      [L(v)]_i
      -
      [L(v)]_{i+1}
      > 0.
      $
      By assumption, we have $p_i \ge p_{i+1}$. If we have $p_i > p_{i+1}$, then 
      \[
        \underbrace{
      (p_{i+1} - p_i)
    }_{< 0}
      (
    \underbrace{
      [L(v)]_i
      -
      [L(v)]_{i+1}
    }_{>0}
      )
      < 0
      \]
      which is a contradiction. Hence, we must have $p_i = p_{i+1}$.
      Repeating the proof with the update $v \gets \tau v$, we obtain a composition of transpositions
      \[\sigma := \sigma_{(i_1,i_1+1)} 
\sigma_{(i_2,i_2+1)} 
\cdots
\sigma_{(i_m,i_m+1)} 
\] such that $\sigma v \in \mathbb{R}^{k}_{\downarrow}$ and $\sigma p = p$.
Finally,
\[
  \underline L(p)
  =
  \langle p, L(v)\rangle
  =
  \langle p, \sigma' \sigma L(v)\rangle
  =
  \langle \sigma p, L( \sigma v)\rangle
  =
  \langle p, L( \sigma v)\rangle
\]
implies that $\sigma v \in \Gamma(p)$.
\end{proof}

\begin{lemma}
  \label{lemma: argmax is Sigma t-equivariant}
  Let $\sigma \in \Perm{k}$ and $v \in \mathbb{R}^k$. Then $\argmax \sigma v = \sigma^{-1}(\argmax v)$.
\end{lemma}
\begin{proof}
  Let $M = \max v = \max \sigma v$.
  \begin{align*}
\argmax \sigma v 
&=
\{j \in [k]: [\sigma v]_j = M\}
\\
&=
\{j \in [k]: [v]_{\sigma(j)} = M\}.
  \end{align*}
  On the other hand,
  \begin{align*}
    \sigma^{-1}
    (\argmax v)
    &=
    \{j \in [k]: \sigma(j) \in \argmax v\}
    \\
    &=
    \{j \in [k]: [v]_{\sigma(j)} = M\}\\
    &=
\argmax \sigma v 
  \end{align*}
  as desired.
\end{proof}

\begin{lemma}
  \label{lemma: argmax v is contained in S1}
  Let $p \in \Delta^k_\downarrow$ be such that $\max p > \frac{1}{k}$.
  Let $v \in \Gamma(p)$, then there exists $\mathbf{S} = (S_1,\dots, S_l) \in \gamma(p)$ such that
  $\argmax v \subseteq S_1$.
\end{lemma}
\begin{proof}
  Recall by definition, $v \in \Gamma(p)$ if and only if $\underline{L}(p) = \langle p, L(v)$.
  We first claim that $v$ is not a scalar multiple of the all ones vector.
  Suppose it is, then $\underline{L}(p) = \langle p, L(v) \rangle
  =\langle p , L(0)\rangle$ by \cref{lemma: L is translation invariant}, which implies that $0 \in \Gamma(p)$.
  Now, by \cref{lemma: the trivial ordered partition}, we have $0 \not \in \Gamma(p)$ since $\min p < \frac{1}{k}$ by the assumption that $\max p > \frac{1}{k}$.
  This is a contradiction.
  Hence, the claim is proved.

  Next, let $n = |\argmax v|$. By our claim that $v$ is non-constant, we have that $n \in [k-1]$.
  Let $\sigma \in \Perm{k}$ be such that $\sigma v \in \mathbb{R}^k_\downarrow$.
  Thus, by construction, we have $\argmax v = [n]$.
  Hence, we have, by \cref{lemma: argmax is Sigma t-equivariant},
  \[
    [n]=
    \argmax \sigma v
    =
    \sigma^{-1}(\argmax v)
  \]
  or, equivalently, $\argmax v = \sigma([n])$.
  Since $n = |\argmax v| \in [k-1]$, $v$ is feasible for the right hand side of 
  \cref{equation: WW hinge n-bayes risk}.
  Thus, we have
  \[
    \underline L(p) = \underline L^n(p).
  \]
  By \cref{lemma: L bayes risk is optimized over integers}
    \begin{equation}
      \label{equation: definition of w star}
      \underline L^n(p) = \min_{w \in \mathcal{C}_{\mathbb{Z}}\,:\, w_n = 0} \langle p, L(w) \rangle.
    \end{equation}
    Let $w^*$ be a minimizer of the above optimization.
    Since $w^* \in \mathcal{C}_{\mathbb Z}$, consider $\mathbf{S} = (S_1,\dots, S_l):= \tilde \psi(w^*)$.
    Hence, by the definition of $\tilde \psi$, we have that $S_1 = \argmax w^*$.
    Note that
\begin{align*}
  \underline L(p)
  =
  \underline{L}^n(p)
  &=
  \langle p, L(w^*) \rangle
  \\
  &=
  \langle p, L(\varphi(\mathbf{S})) \rangle
  \quad \because \mbox{\cref{main-theorem: bijection between OP representations}}
  \\
  &=
  \langle p, \ell(\mathbf{S}) \rangle
  \quad \because \mbox{\cref{main-theorem: L and ell inner risk identity}}
  \\
  &=
  \langle \sigma p, \ell(\mathbf{S}) \rangle
  \quad \because \sigma p =p\mbox{ by \cref{lemma: monotonic non-increasing p}}
  \\
  &=
  \langle p, \sigma' \ell(\mathbf{S}) \rangle
  \\
  &=
  \langle p, \ell(\sigma \mathbf{S}) \rangle
\quad \because \mbox{\cref{lemma: ell is Sigma equivariant}.}
\end{align*}
Putting it all together, we have
\[
  \langle p, \ell(\sigma \mathbf{S}) \rangle
  =
  \underline{L}(p)
  =\underline{\ell}(p)
\]
where the second equality follows from 
\cref{main-lemma: L and ell bayes risk are equal}.
This proves that $\sigma \mathbf{S} \in \gamma(p)$.
Note that since $w^*$ is feasible for the optimization on the right hand side of \cref{equation: definition of w star}, we have $\argmax w^* = \{i \in [k] : w^*_i = 0\} \supseteq [n]$.
Furthermore, recall that $S_1 = \argmax w^*$.
Putting it all together, we have
$\sigma (S_1) \supseteq \sigma([n]) = \argmax v$.
Thus, $\sigma(\mathbf{S})$ satisfies the desired conditions.
\end{proof}

\begin{lemma}
  \label{lemma: gamma is Sigma equivariant}
  For all $p \in \Delta^k$ and $\sigma \in \Perm{k}$, we have
  \begin{align}
    \label{lemma: S equation}
    \mathbf{S} \in \gamma(\sigma p) &\iff \sigma \mathbf{S} \in \gamma(p),\\
    \label{lemma: v equation}
    v \in \Gamma(\sigma p) &\iff \sigma' v \in \Gamma(p).
  \end{align}
\end{lemma}
\begin{proof}
  We first prove 
    \cref{lemma: S equation}.
  Let $\mathbf{S} \in \gamma(\sigma p)$.  Then
  \begin{align*}
    \underline \ell(\sigma p)
    &=
    \langle \sigma p, \ell(\mathbf{S})
    \rangle
    \\
    &=\langle p, \sigma'\ell(\mathbf{S})
    \rangle
    \\
    &=\langle p, \ell(\sigma \mathbf{S})
    \rangle
    \quad \because \mbox{\cref{lemma: ell is Sigma equivariant}}
    \\
    & \ge
    \underline \ell(p).
  \end{align*}
  By the same argument, we have $\underline \ell(p) \ge \underline \ell(\sigma p)$.
  Thus, $\underline \ell(p) = \underline \ell(\sigma p)$ and $\sigma \mathbf{S} \in \gamma(p)$. This proves the $\implies$ direction \cref{lemma: S equation}.
  To prove the other direction, we first write $p = \sigma' \sigma p$ and note that
  \[
    \sigma \mathbf{S} \in \gamma(\sigma' \sigma p)
    \implies
    \sigma '\sigma \mathbf{S} \in \gamma(\sigma p)
    \iff
    \mathbf{S} \in \gamma(\sigma p).
  \]
  
    Next, we prove \cref{lemma: v equation}.
  By \cref{lemma: L bayes risk is Sigma invariant}, we have $\underline{L}(\sigma p) = \underline{L}(p)$.
    Let $v \in \Gamma(\sigma p)$, then
  \begin{align*}
    \underline L(p)=
    \underline L(\sigma p)
    &=
    \langle \sigma p, L(v)
    \rangle
    \\
    &=\langle p, \sigma'L(v)
    \rangle
    \\
    &=\langle p, L(\sigma' v)
    \rangle
    \quad \because \mbox{\cref{corollary: L is Sigma equivariant}.}
  \end{align*}
  Thus, $\sigma'v \in \Gamma(p)$.
  This proves the $\implies$ direction of \cref{lemma: v equation}.
  For the other direction, 
  \[
    \sigma' v \in \Gamma(\sigma' \sigma p)
    \implies
    \sigma \sigma' v \in \Gamma(\sigma p)
    \iff
    v \in \Gamma(\sigma p).
  \]
\end{proof}

\subsection{Proof of \cref{main-theorem: argmax v is contained in S1}}
\begin{proof}[Proof of \cref{main-theorem: argmax v is contained in S1}]
  Let $\sigma \in \Perm{k}$ be such that $\sigma p \in \Delta^k_\downarrow$. By \cref{lemma: gamma is Sigma equivariant}, we have $\sigma v \in \Gamma(\sigma p)$. Then by \cref{lemma: argmax v is contained in S1}, there exists $\mathbf{S} = (S_1,\dots, S_l) \in \gamma(\sigma p)$ such that
  $S_1 \supseteq \argmax \sigma v
  =
  \sigma^{-1} (\argmax v)$, where the equality is due to \cref{lemma: argmax is Sigma t-equivariant}.
  Applying $\sigma$, to both side, we have
  $\sigma S_1 \supseteq \argmax v$.
  By \cref{lemma: gamma is Sigma equivariant}, we have $\sigma \mathbf{S} \in \gamma(p)$. Hence, we are done.
\end{proof}

\begin{lemma}
  \label{lemma: argmax p is contained in S1}
  Let $p \in \Delta^k_\downarrow$ be such that $\argmax p = \{1\}$ and $\mathbf{S} = (S_1,\dots, S_l) \in \gamma(p)$. Then $1 \in S_1$.
\end{lemma}
\begin{proof}
  Let $v = \varphi(\mathbf{S})$.
  Since $\mathbf{S}$ is nontrivial, we have $\max v > \min v$.
  By construction, we have $\argmax v = S_1$.
  Hence, if $1 \not \in S_1$, then there exists some $j \in \{2,\dots, k\}$ such that $v_j > v_1$.
  Then \cref{lemma: L is order reversing} implies that $[L(v)]_1 > [L(v)]_j$ and so
  \[
    \langle p, L(v)\rangle 
    -
    \langle p, \sigma_jL(v)\rangle
    =
    (p_1 - p_j)([L(v)]_1 - [L(v)]_j)
    > 0.
  \]
  But  $
  \underline \ell(p)
  =
  \langle p, \ell(\mathbf{S}) \rangle
  =
    \langle p, L(v)\rangle $ and
  \[
    \langle p, \sigma_jL(v)\rangle
    =
    \langle p, L(\sigma_j v)\rangle
    =
    \langle p, L(\sigma_j \varphi(\mathbf{S})\rangle
    =
    \langle p, L(\varphi(\sigma_j \mathbf{S}))\rangle
    =
    \langle p, \ell(\sigma_j \mathbf{S})\rangle
  \]
  Thus, we have
  \[
  \langle p, \ell(\mathbf{S}) \rangle
  -
    \langle p, \ell(\sigma_j \mathbf{S})\rangle
    > 0
  \]
  which contradicts that $\mathbf{S} \in \gamma(p)$.
\end{proof}

\begin{definition}
  A \emph{$\Perm{k}$-invariant property} is a boolean function 
  \begin{equation}\mathcal{B}: \Delta^k \to \{\verb|true|, \verb|false|\}\end{equation} such that
  $\mathcal{B}(p) \implies \mathcal{B}(\sigma p)$ for all $\sigma \in \Perm{k}$ and $p \in \Delta^k$.
  Here, ``$\implies$'' denotes logical implication.
\end{definition}

\begin{lemma}
  \label{lemma: boolean function equivariant extension}
  Let $\mathcal{B}$ and $\mathcal{C}$ be $\Perm{k}$-invariant properties.
  Suppose that for all $p \in \Delta^k_\downarrow$, $\mathcal{B}(p)$ implies $\mathcal{C}(p)$.
  Then for all $p \in \Delta^k$, we have $\mathcal{B}(p)$ implies $\mathcal{C}(p)$.
\end{lemma}
\begin{proof}
  Let $p \in \Delta^k$ be arbitrary.
  Pick $\sigma$ such that $\sigma p \in \Delta^k_\downarrow$. Then
  \[
    \mathcal{B}(p) \implies \mathcal{B}(\sigma p) 
                   \implies \mathcal{C}(\sigma p) 
                   \implies \mathcal{C}(p)
  \]
  where for the first and last implications we used the $\Perm{k}$-invariance property of $\mathcal{B}$ and $\mathcal{C}$, and for the implication in the middle we used the assumption in the lemma.
\end{proof}
\begin{lemma}
  \label{lemma: argmax p = S1 is Sigma equivariant}
  Let $p \in \Delta^k$.
  Consider the statement $\mathcal{B}_{1}(p)$ which returns $\verb|true|$ if and only if
  \begin{equation}
    \mbox{for all $\mathbf{S} \in \gamma(p)$, $|S_1| = 1$ and $S_1 = \argmax p$.}
  \end{equation}
  Then $\mathcal{B}_1$ is a $\Perm{k}$-invariant property.
\end{lemma}
\begin{proof}
  Let $p \in \Delta^k$ and $\sigma \in \Perm{k}$.
  Suppose $\mathcal{B}_1(p)$ is true. We need to show that $\mathcal{B}_1(\sigma' p)$ is true.
  Let $\mathbf{S} \in \gamma(\sigma p)$.
  By \cref{lemma: gamma is Sigma equivariant}, we have $\sigma \mathbf{S} \in \gamma(p)$.
  Since $\mathcal{B}_1(p)$ is true, we have $|\sigma(S_1)| = 1$ and $\sigma(S_1) = \argmax p$.
  Thus, we immediately get that $|S_1| = 1$.
  By \cref{lemma: argmax is Sigma t-equivariant}, we have $S_1 = \sigma^{-1}(\argmax p) = \argmax \sigma p$.
  The two preceding facts is equivalent to $\mathcal{B}_1(p)$ being true, by definition.
\end{proof}

\subsection{Proof of \cref{main-proposition: majority case}}
\begin{proof}[Proof of \cref{main-proposition: majority case}]
  By \cref{lemma: argmax p = S1 is Sigma equivariant} and \cref{lemma: boolean function equivariant extension}, we may assume $p \in \Delta^k_\downarrow$. \cref{lemma: argmax p is contained in S1} implies that $1 \in S_1$.
  If $|S_1| = 1$, then $S_1 = \{1\}$ and the result is proven.
  Below, suppose $|S_1| >1$.
  We define 
  \[
    S_1' = \{1\}, \quad S_1'' = S_1 \setminus \{1\}.
  \]
  Define
  \[
    \mathbf{S}'
    =(S_1',S_1'',S_2,\dots, S_l) \in \mathcal{OP}_k.
  \]
  We claim that $\langle p, \ell(\mathbf{S}')\rangle
  < \langle p, \ell(\mathbf{S})\rangle$.
  Given the claim, we would have a contradiction that $\mathbf{S} \in \gamma(p)$ and so $|S_1| = 1$ must be true.
  Let $Y \sim p$ and define
  \[
\beta:=
\sum_{j = 1}^{l-1}
|S_1 \cup \cdots \cup S_{j+1}|
\Pr(Y \not \in S_1\cup \cdots \cup S_j)
  \]
  Observe that
  \begin{align*}
\langle p, \ell(\mathbf{S}')\rangle
&=
|S_1'| - 1
+
|S_1'\cup S_1''|
\Pr(Y \not \in S_1') + \beta
\\
&=
|S_1| \Pr(Y \ne 1)
+ \beta
\\
&<
\frac{1}{2}|S_1|
+ \beta.
  \end{align*}
  On the other hand, we have
  \[
    \langle p, \ell(\mathbf{S})\rangle
    =
    |S_1| - 1 
    + \beta.
  \]
  Hence, we have
  \begin{align*}
    \langle p, \ell(\mathbf{S})\rangle
    -
    \langle p, \ell(\mathbf{S}')\rangle
    &=
    |S_1| - 1 
    -
|S_1| \Pr(Y \ne 1)
    \\
    &>
    |S_1| - 1 
    -
    \frac{1}{2}|S_1|
    \\
    &=
    \frac{1}{2}|S_1| -1
    \\
    &\ge
    \frac{2}{2} - 1
    \\
    &= 0.
  \end{align*}
  which proves the claim.
\end{proof}

\subsection{Proof of \cref{main-proposition: spiked case}}
\begin{proof}[Proof of \cref{main-proposition: spiked case}]
  Since $\argmax p = \{j^*\}$, we have $
      (\{j^*\},\, [k] \setminus \{j^*\})
      = (\argmax p,\, [k] \setminus \argmax p).
  $
  We check that the statement below
defines a $\Perm{k}$-invariant property:
  \begin{equation}
    \label{equation: SLN condition stated in terms of OP}
    \mbox{``$p$ satisfies $(\argmax p,\, [k] \setminus \argmax p)$ is the unique element of $\gamma(p)$.'' }
\end{equation}
Let $p$ satisfy \cref{equation: SLN condition stated in terms of OP}. 
By 
  \cref{lemma: gamma is Sigma equivariant}, we have
  $\sigma^{-1}(\argmax p,\, [k] \setminus \argmax p)
  $ is the unique element of $\gamma(\sigma p)$.
  By definition,  
  \[\sigma^{-1}(\argmax p,\, [k] \setminus \argmax p)
  =
  (\sigma^{-1}\argmax p,\, \sigma^{-1}([k] \setminus \argmax p)).
  \]
  By 
  \cref{lemma: argmax is Sigma t-equivariant}, we have
  $\sigma^{-1} \argmax p = \argmax \sigma p$.
  Thus, we have
  \[\sigma^{-1}(\argmax p,\, [k] \setminus \argmax p)
  =
  (\argmax \sigma p,\, [k] \setminus \argmax \sigma^{-1}p)
  \]
  is the unique element of $\gamma(\sigma p)$. In other words, $\sigma p$ satisfies \cref{equation: SLN condition stated in terms of OP}, as desired.

Furthermore, 
$
\mbox{``$p$ satisfies the symmetric noise condition.''}
$
is obviously $\Perm{k}$-invariant.
  Hence, by \cref{lemma: argmax p = S1 is Sigma equivariant} and \cref{lemma: boolean function equivariant extension}, we may assume $p \in \Delta^k_\downarrow$.
  Pick $\mathbf{S} = (S_1,\dots, S_l) \in \gamma(p)$.
  \cref{lemma: argmax p is contained in S1} implies that $1 \in S_1$.
  By \cref{main-definition: ordered partitions} of $\mathcal{OP}_k$, we have $l \ge 2$. We first show that $l = 2$ by contradiction.
  Suppose that $l > 2$.
  Define
  $\mathbf{S}' = (S_1',\dots, S_{l-1}')$ where
  \[
    S_1' := S_1,\, \quad S_2' := S_2 \cup S_3, \quad S_j' := S_{j+1}, \, \forall j \in \{3, \cdots, l-1\}.
  \]
  
  Let  $Y \sim p$ and
  \[
    \beta
    :=
    \sum_{j=3}^{l-1}
    |S_1 \cup \cdots \cup S_{j+1}|
    \Pr( Y \not \in S_1 \cup \cdots \cup S_j).
  \]
  Then we have
  \begin{align*}
    \langle p, \ell(\mathbf{S})\rangle
    &=
    |S_1| - 1
    +
    |S_1 \cup S_2| \Pr(Y \not \in S_1)
    \\
    &\qquad
    +
    |S_1 \cup S_2 \cup S_3| \Pr(Y \not \in S_1 \cup S_2)
    + \beta
  \end{align*}
  and
  \begin{align*}
    \langle p , \ell(\mathbf{S}')\rangle
    &=
    |S_1'| - 1
    + 
    |S_1' \cup S_2'| \Pr(Y \not \in S_1')
    \\
    &\qquad +
    \sum_{j=2}^{l-2}
    |S_1' \cup \cdots \cup S_{j+1}'| \Pr(Y \not \in S_1'\cup \cdots \cup S_j')
    \\
    &=
    |S_1| - 1
    + 
    |S_1 \cup S_2\cup S_3| \Pr(Y \not \in S_1)
    \\
    &\qquad +
    \sum_{j=2}^{l-2}
    |S_1 \cup \cdots \cup S_{j+2}| \Pr(Y \not \in S_1\cup \cdots \cup S_{j+1})
    \\
    &=
    |S_1| - 1
    + 
    |S_1 \cup S_2\cup S_3| \Pr(Y \not \in S_1)
    \\
    &\qquad +
    \sum_{j=3}^{l-1}
    |S_1 \cup \cdots \cup S_{j+1}| \Pr(Y \not \in S_1\cup \cdots \cup S_{j})
    \\
    &=
    |S_1| - 1
    + 
    |S_1 \cup S_2\cup S_3| \Pr(Y \not \in S_1) + \beta
  \end{align*}

  Putting it all together, we have
  \begin{align*}
    \langle p, \ell(\mathbf{S})\rangle
    -
    \langle p , \ell(\mathbf{S}')\rangle
    &=
    |S_1 \cup S_2| \Pr(Y \not \in S_1)
    \\
    &\qquad
    +
    |S_1 \cup S_2 \cup S_3| \Pr(Y \not \in S_1 \cup S_2)
    \\
    &\qquad
    - 
    |S_1 \cup S_2\cup S_3| \Pr(Y \not \in S_1)
    \\
    &=
    |S_1 \cup S_2 \cup S_3| \Pr(Y \not \in S_1 \cup S_2)
    \\
    &\qquad
    -
    |S_3| \Pr(Y \not \in S_1).
  \end{align*}
  Define $s_i := |S_i|$ for each $i \in [l]$. Then
  \[
    |S_1 \cup S_2 \cup S_3| \Pr(Y \not \in S_1 \cup S_2)
    =
    (s_1 + s_2 + s_3)
    (k - s_1 - s_2)
    \frac{1- \alpha}{k-1}
  \]
  and
  \[
    |S_3| \Pr(Y \not \in S_1)
    =
    s_3
    (k - s_1)
    \frac{1- \alpha}{k-1}.
  \]
  Now, we have
  \begin{align*}
    &
    (s_1 + s_2 + s_3)
    (k - s_1 - s_2)
    -
    s_3
    (k - s_1)
    \\
    &=
    ((s_1 + s_2) + s_3)
    ((k - s_1) - s_2)
    -
    s_3
    (k - s_1)
    \\
    &=
    (s_1+s_2)(k-s_1)
    -s_2(s_1+s_2)
    -s_2s_3
    \\
    &=
    (s_1+s_2)k
    -(s_1+s_2)^2
    -s_2s_3
    \\
    &\ge
    (s_1+s_2)(s_1+s_2+s_3)
    -(s_1+s_2)^2
    -s_2s_3
    \\
    &=
    s_1s_3
  \end{align*}
  where for the inequality, we used the fact that $k \ge s_1 + s_2 +s_3$.
  Finally, we now get a contradiction of the optimality of $\mathbf{S}$:
  \[
    |S_1 \cup S_2 \cup S_3| \Pr(Y \not \in S_1 \cup S_2)
    -
    |S_3| \Pr(Y \not \in S_1)
    \ge
    s_1s_3 \frac{1- \alpha}{k-1}
    > 0
  \]
  implies
  \[
    \langle p, \ell(\mathbf{S})\rangle
    -
    \langle p , \ell(\mathbf{S}')\rangle
    >0.
  \]
  This proves the claim that if $\mathbf{S} = (S_1,\dots, S_l) \in \gamma(p)$, then $l=2$ and so $\mathbf{S} = (S_1, [k] \setminus S_1)$.
  Next, we show that $S_1 = \{1\}$.
  We already have shown that $1 \in S_1$.
  We calculate
  \begin{align*}
    \langle p, \ell((S_1, [k] \setminus S_1))\rangle
    &=
    |S_1| - 1
    +
    k \Pr(Y \not \in S_1)
    \\
    &=
    |S_1| - 1
    +
    k
    (k-|S_1|)
    \left(
    \frac{1 - \alpha}{k-1}
    \right)
    \\
    &=
    |S_1|
    \left(
      1 - 
    k
    \left(
    \frac{1 - \alpha}{k-1}
    \right)
    \right)
    +C
  \end{align*}
  where $C = -1 +k^2\left(
    \frac{1 - \alpha}{k-1}
    \right)$ does not depend on $|S_1|$.
    To prove that $|S_1| =1$, by minimality of $\mathbf{S}$ it suffices to show that 
    \[
      1 - 
    k
    \left(
    \frac{1 - \alpha}{k-1}
    \right)
    >0.
    \]

    To see this, note that
    \begin{align*}
      1 >
    k
    \left(
    \frac{1 - \alpha}{k-1}
    \right)
     \iff &\frac{1}{k} >
    \frac{1 - \alpha}{k-1}
    \\
     \iff &\frac{k-1}{k} = 1 - \frac{1}{k}>
     1-\alpha
    \\
     \iff &
    \alpha > \frac{1}{k}
    \end{align*}
    where the last line is part of our assumption in the lemma statement.  
\end{proof}
\newpage

\section{Derivation of the figures}
\label{section: derivation of the figures}
We discuss how \cref{main-figure: Bayes optimal classifier of the OP loss,main-figure: bayes decision regions of WW vs CS} are obtained.

\subsection{\cref{main-figure: Bayes optimal classifier of the OP loss}} 
\label{section: figure 1 generating code}
When $k=3$, there are $12$ nontrivial ordered partitions.
Below, we represent $\mathcal{OP}_3$ vectorially in $\mathbb{R}^3$ using \cref{main-theorem: bijection between OP representations}:
\begin{lstlisting}[language=Matlab]
OPk = [-2 -2 -1 -1 -1 -1  0 -1  0  0  0  0 ;
        0 -1  0  0  0 -1  0 -2 -1 -1 -1 -2 ;
       -1  0  0 -1 -2  0 -1  0  0 -1 -2 -1 ]
\end{lstlisting}
Every column of the matrix \verb|OPk| is a nontrivial ordered partition, e.g., the first column
$
  \begin{bmatrix}
    -2\\
    0\\
    -1
  \end{bmatrix}
  \mapsto 
  2|3|1.
$
Consider the following matrix whose columns are $\ell(\mathbf{S}) = L^{WW}(\varphi(\mathbf{S})) \in \mathbb{R}^3_+$ where $\ell$ is the ordered partition loss and $\mathbf{S}\in \mathcal{OP}_3$.
\begin{lstlisting}[language=Matlab]
ell = [ 5  5  4  3  2  3  1  2  1  0  0  0 ; 
        0  2  1  0  0  3  1  5  4  3  2  5 ; 
        2  0  1  3  5  0  4  0  1  3  5  2 ]
\end{lstlisting}
For example, the first column  of \verb|ell| is the result of applying $L^{WW}: \mathbb{R}^k \to \mathbb{R}^k_+$ to the first column of \verb|OPk|
$, i.e.,
  \begin{bmatrix}
    5\\
    0\\
    2 
  \end{bmatrix}
  =
  L^{WW}\left(
  \begin{bmatrix}
    -2\\
    0\\
    -1
  \end{bmatrix}
  \right)
  = \ell^{\mathcal{OP}}(
  2|3|1).
$
Finally, to get the region in \cref{figure: matlab generated figure} labelled by ``$2|3|1$'', we plot the $(p_2,p_3)$ coordinates of the following polytope:
\begin{equation*}
  \mathrm{Reg}(2|3|1) :=
  \{ p \in \Delta^3 : \langle p, \ell(2|3|1) - \ell(\mathbf{S})\rangle \le 0,\, \forall \mathbf{S} \in \mathcal{OP}_3,\, \mathbf{S} \ne 2|3|1\}.
\end{equation*}
Repeat this procedure for all of $\mathcal{OP}_3$, we obtain \cref{figure: matlab generated figure}.

\begin{figure}[H]
  \centering
  \includegraphics[width = 0.75\textwidth]{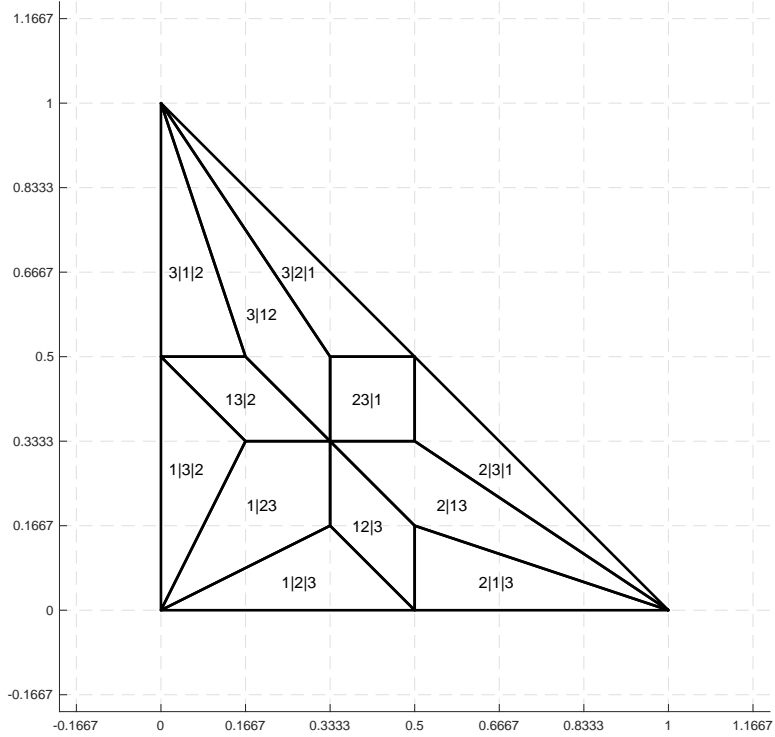}
  \caption{
    Each polygonal region is the polytope
    $\mathrm{Reg}(\mathbf{S})$ projected onto its last two coordinates overall $\mathbf{S} \in \mathcal{OP}_3$.
}
\label{figure: matlab generated figure}
\end{figure}

\subsection{\cref{main-figure: bayes decision regions of WW vs CS}} 
\label{section: figure 2 detailed derivation}
For the left panel of \cref{main-figure: bayes decision regions of WW vs CS}, we compute $\Omega_{L^{WW}}$
\[
  \Omega_{L^{WW}}
  :=
  \{ p \in \Delta^k : 
    |\argmax p| = 1,\,
  \argmax v = \argmax p,\, \forall v \in \Gamma_{L^{WW}}(p)\}.
\]
Thus, the region in light gray in the left panel of \cref{main-figure: bayes decision regions of WW vs CS} is the union of the polygons of \cref{main-figure: Bayes optimal classifier of the OP loss} labelled by an ordered partition whose the top bucket has $2$ elements.
This characterize $\Omega_{L^{WW}}$ up to a set of Lebesgue measure zero.

For the right panel,
consider $v \in \Gamma_{L^{CS}}(p)$.
\citet[Lemma 4]{liu2007fisher} states that if $\max p < 1/2$, then $v = (0,0,0)$.
Furthermore, if $\max p > 1/2$, then $\argmax v = \argmax p$.
This characterize $\Omega_{L^{CS}}$ up to a set of Lebesgue measure zero.

\end{document}